\documentclass[12pt]{article}
\usepackage{amsthm,amsmath,amssymb,bbm,bm}
\usepackage{multirow}
\usepackage{xr-hyper}
\usepackage[pdftex]{graphicx}
\usepackage{subfigure}
\usepackage{wrapfig}
\usepackage{CJK}
\usepackage{array}
\usepackage{url}
\usepackage{booktabs}
\usepackage{algorithm}
\usepackage[dvipsnames]{xcolor}
\usepackage{algorithmic}
\usepackage{placeins}
\usepackage{mathtools}
\usepackage{romannum}
\usepackage{mathrsfs}
\usepackage{dsfont}
\usepackage{natbib}
\usepackage{relsize}
\usepackage{rotating}
\usepackage{enumitem}
\usepackage{setspace}
\usepackage{tikz}
\usepackage{multirow}
\usepackage{tikz}
\usepackage{caption}
\usepackage{subcaption}
\usetikzlibrary{shapes,shapes.multipart}
\usetikzlibrary{arrows}
\usepackage[page,header]{appendix}
\usepackage{titletoc}
\usepackage{titling}

\usepackage[utf8]{inputenc} 
\usepackage[T1]{fontenc}    
\usepackage{url}            
\usepackage{booktabs}       
\usepackage{amsfonts}       
\usepackage{nicefrac}       
\usepackage{microtype}      
\usepackage{xcolor} 
\usepackage{amsmath,amssymb,amsthm, float}
\usepackage{graphicx}
\usepackage{mathtools}
\usepackage{float}
\usepackage{subcaption}
\newtheorem{theorem}{Theorem}
\newtheorem{assumption}{Assumption}
\newtheorem{proposition}{Proposition}

\def\T{\mathcal{T}}

\usepackage{multibib}

\newcites{SM}{References}

\usepackage[colorlinks]{hyperref}
\hypersetup{
    colorlinks,
    linkcolor={red},
    filecolor={red}, 
    urlcolor={blue},
    citecolor={blue}
}

\makeatletter
\newcommand*{\addFileDependency}[1]{
  \typeout{(#1)}
  \@addtofilelist{#1}
  \IfFileExists{#1}{}{\typeout{No file #1.}}
}
\makeatother

\usepackage{etoolbox}
\newcommand{\zerodisplayskips}{%
  \setlength{\abovedisplayskip}{6.5pt}%
  \setlength{\belowdisplayskip}{6.5pt}%
  \setlength{\abovedisplayshortskip}{6.5pt}%
  \setlength{\belowdisplayshortskip}{6.5pt}}
\appto{\normalsize}{\zerodisplayskips}
\appto{\small}{\zerodisplayskips}
\appto{\footnotesize}{\zerodisplayskips}

\makeatletter
\newcommand{\vast}{\bBigg@{3}}
\makeatother

\usepackage{chngcntr}
\counterwithin{equation}{section}
\definecolor{ao(english)}{rgb}{0.0, 0.5, 0.0}

\theoremstyle{plain}

\counterwithin{thm}{section}

\counterwithin{proposition}{section}

\counterwithin{corollary}{section}

\counterwithin{theorem}{section}

\def\##1\#{\begin{align}#1\end{align}}
\def\$#1\${\begin{align*}#1\end{align*}}

\def\beq#1\eeq{\begin{equation}#1\end{equation}}
\def\baa#1\eaa{\begin{eqnarray}#1\end{eqnarray}}
\def\bal#1\eal{\begin{align}#1\end{align}}

\DeclareMathOperator*{\argmin}{arg\,min}

\def\T{\text{T}}

\newcommand{\blind}{1}

\usepackage[margin=0.9in]{geometry}
\begin{document}

\setcounter{page}{1}
\pagenumbering{arabic}


\if1\blind
{ \title{\bf Offline Policy Evaluation via Mixed Bellman Residuals and Adaptive Critic Representations}}

\author{
Amitakshar Biswas \quad Yuhan Li \quad Ruoqing Zhu
\\[6pt]
\normalsize Department of Statistics, University of Illinois at Urbana-Champaign
}\fi

\if0\blind
{
    \title{\bf Offline Policy Evaluation via Mixed Bellman Residuals and Adaptive Critic Representations}
} \fi

\newpage

\date{}

\maketitle
\vspace{-10mm}
\begin{abstract}
Evaluating a target policy using data generated by a different behavior policy remains a fundamental challenge in reinforcement learning. While most existing work relies on the standard one-step Bellman residual, we consider a convex combination of one-step and two-step residuals with a fixed mixing weight. In ideal settings, this mixed Bellman formulation can provide a natural bias--variance trade-off between approximation error under a restricted value-function class and the increased variance arising from multi-step importance weighting. To solve this mixed residual optimization, we adopt a minimax formulation involving a critic function. Unlike standard approaches that rely on a fixed functional class, we construct a data-dependent critic representation using predicted future feature directions which effectively induces a kernel adapted to the underlying transition dynamics. This allows the critic to focus on directions that are most relevant for the estimated Bellman error. To control overfitting, we use sample splitting to construct the critic and estimate the value function on separate data subsets. Simulation studies and MetaWorld tasks illustrate the effect of the mixing parameter and show that intermediate residual combinations can improve value estimation in challenging settings.
\end{abstract}
\vfill

\newpage


%

\newcommand{\fix}{\marginpar{FIX}}
\newcommand{\new}{\marginpar{NEW}}

\section{Introduction}\label{intro}

Reinforcement learning \citep{sutton2018reinforcement} has become a central framework for sequential decision-making, with applications ranging from healthcare \citep{murphy2001marginal} and personalized medicine \citep{luckett2019estimating} to robotics \citep{levine2020offline} as well as autonomous driving \citep{zhu2020safe}. In many of these settings, directly interacting with the environment to collect new data is costly, unsafe, or ethically problematic. This has led to strong interest in \textit{offline} reinforcement learning, where the goal is to learn from a fixed, pre-collected dataset generated by some behavior policy, without any further environment interaction \citep{precup2000eligibility, levine2020offline}.

A central problem in offline RL is off-policy evaluation (OPE), which
estimates a target policy from data generated by another policy. Existing
approaches include value-based, importance-sampling, and doubly robust
methods \citep{le2019batch,liao2021off,liu2018breaking,xie2019towards,
uehara2020minimax,kallus2022doubly}. Multi-step Bellman operators are also
closely related to classical $n$-step TD and $\lambda$-return methods
\citep{sutton1988learning,precup2000eligibility}. Our approach instead mixes
one-step and two-step Bellman residuals within a single minimax OPE objective
and combines them with a critic learned from predicted future features.

A core challenge shared by many of these methods is the Bellman residual minimization (BRM) approach \citep{baird1995residual}. BRM seeks an approximate fixed point of the
Bellman operator \citep{baird1995residual,blackwell1965discounted}, but direct
squared-residual minimization suffers from double sampling and Bellman error
can be poorly aligned with value error
\citep{double_sampling,fujimoto2022itrustyoubellman}. Fixed-point methods
avoid double sampling but can be unstable with function approximation and
off-policy data \citep{tsitsiklis1997analysis,gordon1995stable,
boyan1995generalization}. Minimax approaches instead test the residual against
critic functions, yielding moments that can be estimated directly from
observed transitions \citep{uehara2020minimax,kernelloss}.

A separate difficulty is with regard to the choice of Bellman operator. Most existing OPE methods rely on the standard one-step operator, which contracts with factor $\gamma$. When the value function is restricted to a misspecified function class, as is common in practice, approximation error can be significant. Even when the state is Markov, a restricted value representation may not capture how values evolve across the transition dynamics. The two-step Bellman operator contracts with factor $\gamma^2$ and provides a different approximation target within the same restricted value class. Partial observability can create an additional challenge; we study this only as an empirical robustness setting in Env II rather than as part of the MDP theory. However, estimating the two-step residual from off-policy data requires products of importance weights across consecutive transitions, which can increase variance. This motivates mixing one-step and two-step residuals rather than fixing a single Bellman horizon.

In this paper, we propose \textbf{Mixed-AC}, a minimax OPE framework with two main
components. First, we introduce a mixed Bellman residual that combines
one-step and two-step information through the parameter $\alpha\in[0,1]$ so the corresponding
operator remains a contraction with factor
$(1-\alpha)\gamma+\alpha\gamma^2$. Second, we develop an adaptive critic that
augments current-state features with conditional predictions of one-step and
two-step future features learned from the observed transitions. We combine these components in a tractable minimax estimator,
establish a finite-sample stability bound, and evaluate the method on both
simulated and MetaWorld environments.

\section{Background and Notation}

We model a dynamic treatment regime in the infinite-horizon setting by a Markov decision process (MDP; \citealp{puterman1994mdp}), denoted by $(\mathcal{S}, \mathcal{A}, \mathbf{P}, R, \gamma)$, where $\mathcal{S}$ is the (continuous) state space, $\mathcal{A}$ is the action (or treatment) space, $\mathbf{P}(\cdot | s, a)$ denotes the unknown transition kernel, $R:\mathcal{S}\times\mathcal{A}\to\mathbb{R}$ is the reward function, and $\gamma \in (0,1)$ is the discount factor that balances immediate and future rewards. We assume $(\mathcal{S}, \mathcal{B}(\mathcal{S}))$ and $(\mathcal{A}, \mathcal{B}(\mathcal{A}))$ are measurable spaces, and all transition kernels are defined with respect to appropriate base measures.

A policy $\pi(\cdot|s)$ specifies a distribution over actions at state $s$.
We consider a behavior policy $\pi_b$ that generates the data and a target
policy $\pi$ that we want to evaluate. We assume the batch data consist of $n$
i.i.d. trajectories $\mathcal D_n=\{\mathcal D^i\}_{i=1}^n$, where
\[
\mathcal D^i=\{(S_t^i,A_t^i,R_t^i,S_{t+1}^i)\}_{t=0}^{T}.
\]
Thus each logged trajectory contains $T+1$ transitions; the horizon is finite
in the observed data, while the target value is defined through
infinite-horizon discounting. For a fixed target policy $\pi$, the state-value function is defined as
\begin{equation}
V^{\pi}(s)
=
\mathbb{E}_{\pi}\Bigg[
\sum_{k=0}^{\infty}\gamma^{k} R_{t+k}
\;\Big|\; S_t=s
\Bigg],
\end{equation}
where expectation is taken with respect to the target policy $\pi$ and the true transition kernel $\mathbf{P}$. Our goal is to estimate this value function from batch data $\mathcal{D}$. The value function $V^{\pi}$ satisfies a recursive relationship through the Bellman evaluation operator. Specifically, for any bounded function $V:\mathcal{S}\to\mathbb{R}$, the one-step Bellman operator associated with $\pi$ is defined as
\begin{equation}\label{one-step}
(\mathcal{T}_{\pi}V)(s)
=
\mathbb{E}_{a\sim\pi(\cdot|s),\, s'\sim \mathbf{P}(\cdot|s,a)}
\left[
R(s,a) + \gamma V(s')
\right].
\end{equation}
The operator $\mathcal{T}_{\pi}$ maps a given value function to the expected immediate reward plus the discounted value of the next state under the target policy. It is well known that the true value function $V^{\pi}$ is the unique fixed point of $\mathcal{T}_{\pi}$, i.e., $V^{\pi}=\mathcal{T}_{\pi}V^{\pi}$ \citep{blackwell1965discounted, bertsekas1997nonlinear}.

It is often useful to consider multi-step Bellman operators. In particular, the two-step Bellman operator is defined as the composition of the one-step operator with itself, $\mathcal{T}_{\pi}^{(2)}V \;\equiv\; \mathcal{T}_{\pi}(\mathcal{T}_{\pi}V),$ which can be written explicitly as
\begin{align}\label{two-step}
(\mathcal{T}_{\pi}^{(2)}V)(s)
=
\mathbb{E}_{\substack{a\sim\pi(\cdot|s),\,s'\sim\mathbf{P}(\cdot|s,a),\\
a'\sim\pi(\cdot|s'),\,s''\sim\mathbf{P}(\cdot|s',a')}}
\left[
R(s,a)+\gamma R(s',a')+\gamma^{2} V(s'')
\right],
\end{align}

While the two-step operator $\mathcal{T}_{\pi}^{(2)}$ shares the same fixed point as the one-step operator and both act as contraction mappings under the supremum norm, $\mathcal{T}_{\pi}^{(2)}$ enjoys a stronger contraction factor of $\gamma^{2}$ compared to $\gamma$ for $\mathcal{T}_{\pi}$.

\section{Methodology}\label{sec:method}

We first introduce the mixed Bellman residual, then formulate the minimax
objective and the adaptive critic class, and finally in Section~\ref{sec:estimation}, we describe the empirical estimator.

\subsection{Mixed Bellman Residual}\label{sec:mixed_residual}

Off-policy value estimation with function approximation is challenging for
several distinct reasons. Semi-gradient fixed-point methods can be unstable
under the combination of function approximation, bootstrapping, and
off-policy data, while direct Bellman residual minimization avoids this
particular divergence mechanism but suffers from double sampling and from
the imperfect relationship between Bellman residual and value error
\citep{sutton2018reinforcement,baird1995residual,
fujimoto2022itrustyoubellman}.

One way to use more trajectory information is through multi-step residuals.
Within a Markov model, the two-step operator provides a different
approximation target when the value function is restricted to a finite or
misspecified class. It can therefore capture transition information that is
not well represented by a one-step approximation. This comes at a cost:
two-step off-policy estimation requires products of importance weights and
uses more trajectory noise, which can increase variance.

We therefore combine one-step and two-step Bellman residuals. For any value function $V:\mathcal{S}\to\mathbb{R}$, define
$\delta_V^{(1)} \equiv \mathcal{T}_{\pi}V - V$ and $\delta_V^{(2)} \equiv \mathcal{T}_{\pi}^{(2)}V - V$ and the mixed residual
\begin{equation}
     \delta_V^{\alpha} \equiv (1-\alpha)\, \delta_V^{(1)} + \alpha\, \delta_V^{(2)},
\end{equation}
where $0 \leq \alpha \leq 1$. The parameter $\alpha$ indexes a family of OPE estimators that interpolate between one-step and two-step Bellman residual minimization. Different values of $\alpha$ place different weights on the two residuals, and the preferred choice depends on the underlying dynamics, function class, and available data. Selecting the optimal member of this family using only offline data is itself a challenging model-selection problem, shared by many OPE procedures \citep{tang2021modelselection,liu2025opeselection}. In this work, we primarily focus on the study of the behavior of this estimator family across a range of values of $\alpha$.

We focus on one- and two-step residuals as the simplest extension beyond the
standard one-step operator. This introduces additional trajectory information
while limiting importance weighting to products of two ratios and requiring
only two auxiliary conditional-feature predictors; longer-horizon mixtures
are possible but are not studied here.

In practice, the value function is estimated within a restricted function class, which introduces approximation error. One aspect of this error is governed by the contraction behavior of the Bellman operator. The one-step operator $\T_\pi$ contracts with factor $\gamma$, while the two-step operator $\T_\pi^{(2)}$ contracts with factor $\gamma^2$. Define the mixed Bellman operator
\begin{equation}\label{mixedop}
\T_\pi^{\alpha}
=
(1-\alpha)\,\T_\pi + \alpha\,\T_\pi^{(2)},
\quad \alpha \in [0,1].
\end{equation}

The following result establishes the contraction property of the mixed operator.

\begin{proposition}[Contraction Property]\label{contraction}
Let $\T_\pi$ and $\T_\pi^{(2)}$ denote the corresponding one-step and two-step Bellman operators associated with a fixed policy $\pi$, and let $\T_\pi^{\alpha}$ be defined in \eqref{mixedop}. 
Then for any bounded value functions $V_1,V_2:\mathcal{S}\to\mathbb{R}$,
\begin{equation}
\|\T_\pi^{\alpha}V_1 - \T_\pi^{\alpha}V_2\|_\infty
\le
((1-\alpha)\gamma + \alpha\gamma^2)\,
\|V_1 - V_2\|_\infty.
\end{equation}
In particular, if $\alpha > 0$ and $0<\gamma<1$, then $
(1-\alpha)\gamma + \alpha\gamma^2 < \gamma,
$ so $\T_\pi^{\alpha}$ is a contraction with a strictly smaller Lipschitz constant than $\T_\pi$.
\end{proposition}

The proof follows from the triangle inequality and is provided in Appendix~\ref{contractionproof}. Let
$\gamma_\alpha=(1-\alpha)\gamma+\alpha\gamma^2$. Since $V^\pi$ is the fixed point of $\T_\pi^\alpha$, contraction gives the deterministic residual bound
\begin{equation}
\|V-V^\pi\|_\infty
\leq
\frac{1}{1-\gamma_\alpha}
\|\T_\pi^\alpha V-V\|_\infty
=
\frac{1}{(1-\gamma)(1+\alpha\gamma)}
\|\delta_V^\alpha\|_\infty.
\end{equation}
Thus, at the population sup-norm level, increasing $\alpha$ reduces the
worst-case amplification of mixed residual error. This does not by itself
guarantee lower statistical error for the restricted minimax estimator,
because the critic class, value class, data distribution, and multi-step
importance-weight variability also matter.

To evaluate a candidate value function, let $\mu_b$ denote the distribution
of current states in the logged data. For pooled finite-horizon trajectories,
this is the mixture of the behavior-state distributions over the included
time points. For a measurable function $h$, define
\begin{equation}
    \mathcal{L}(V,h)
    =
    \mathbb{E}_{S\sim\mu_b}
    \left[h(S)\delta_V^\alpha(S)\right].
\end{equation} 
In the next section, we study how this loss can be used to characterize the target value function and motivate a minimax formulation based on an adaptive choice of critic functions.

\subsection{Minimax Formulation}\label{sec:minimax}

For a given function $h$, we have defined a loss that measures the mixed Bellman residual of $V$ weighted by $h$. A natural question is why we consider a critic-weighted Bellman residual rather than the squared Bellman error. The reason is that the squared Bellman error involves an expectation of the form
$\mathbb{E}[(\mathbb{E}[R(s,a)+\gamma V(s')-V(s)|s,a])^2]$,
where the inner expectation is taken with respect to the unknown transition kernel $\mathbf{P}(\cdot|s,a)$. In offline settings, this inner expectation cannot be evaluated exactly and must be approximated from data, which leads to the \textit{double sampling} issue \citep{double_sampling, wang2014stochasticcompositionalgradientdescent}. It is also well known that the Bellman error is a poor replacement for value error and its magnitude hides bias \citep{fujimoto2022itrustyoubellman}. By contrast, the critic-weighted loss is linear, since the Bellman residual appears inside a single expectation and can be estimated directly from observed transition samples.

We first note the population identification property that motivates the
minimax objective. This result provides the foundation for the critic
formulation used below.

\begin{assumption}[Bounded rewards]\label{ass:bounded_reward}
There exists $R_{\max}>0$ such that
$
|R(s,a)|\le R_{\max}
$
for all $(s,a)\in\mathcal{S}\times\mathcal{A}$.
\end{assumption}

Together with $0<\gamma<1$, Assumption~\ref{ass:bounded_reward} implies that
$V^\pi$ is bounded.

\begin{theorem}[Identification]\label{thm:identification}
Suppose Assumption~\ref{ass:bounded_reward} holds. Let \(V\) be a bounded
measurable value function such that
\(\delta_V^\alpha \in L^1(\mu_b)\). Then the following statements are equivalent:
\begin{itemize}
\item[(i)] \(\mathcal{L}(V,h)=0\) for every bounded measurable critic function \(h\).

\item[(ii)] \(\delta_V^\alpha(s)=0\) for \(\mu_b\)-almost every
\(s\in\mathcal{S}\).
\end{itemize}
In particular, $V^\pi$ satisfies
$\mathcal{L}(V^\pi,h)=0$ for every bounded measurable critic function $h$.
Moreover, if $\delta_V^\alpha$ is continuous and $\mu_b$ has full
support on $\mathcal S$, then $\mathcal{L}(V,h)=0$ for every bounded
measurable $h$ implies $V=V^\pi$.
\end{theorem}

Theorem~\ref{thm:identification} gives the population justification for
testing the mixed Bellman residual against critic functions. It uses the
full class of bounded critics, while the estimator below uses a smaller
data-dependent class.

We now introduce the minimax framework
\citep{uehara2020minimax,duan,jiang,lihong}. For the batch of $n$ independent trajectories defined above, Theorem~\ref{thm:identification} gives
\begin{equation*}
\mathbb{E}_{S \sim \mu_b}\!\left[\,h(S)\,\delta_{V^\pi}^{\alpha}(S)\,\right] = 0,
\end{equation*}
where $V^\pi$ is the true value function and $h$ is any bounded measurable function. This motivates the minimax formulation
\begin{equation}
\argmin_{V \in \mathcal{V}}
\max_{h \in \mathcal{F}}
\;\mathcal{L}^2(V,h),
\end{equation}
where $\mathcal V$ and $\mathcal F$ denote the value and critic classes,
respectively. The remaining question is how to choose a critic class that is both rich enough to detect Bellman violations and structured enough to produce a tractable estimator. In the next section, we introduce an adaptive critic class learned from the transition structure of the data.

\subsection{Adaptive Critic Representation}
\label{sec:adaptive_critic}

A key observation is that although the population mixed Bellman residual
$\delta_V^{\alpha}(s)$ is a function of the current state, its empirical
estimation uses rewards and future states observed along the trajectory.
In particular, the one-step and two-step terms involve value evaluations at
$S_i'$ and $S_i''$, while off-policy estimation additionally uses importance
weights over the observed actions. Existing minimax OPE methods often restrict
the critic to a reproducing kernel Hilbert space (RKHS) generated by a fixed
universal kernel on the \emph{current} state space \citep{kernelloss}. While
tractable, this uses the same critic representation across different MDPs. We
instead construct an \emph{adaptive} critic from current-state predictors of
future transition features.

Let $\phi:\mathcal{S}\to\mathcal{H}$ be a fixed feature map on the state
space, where $\mathcal{H}$ is a Hilbert space. Let $\mathbb{E}_b$ denote
expectation under the behavior trajectory distribution induced by the behavior
policy $\pi_b$. Define the one-step and two-step conditional feature means as
\[
m_1^b(s)
=
\mathbb{E}_b\!\left[\phi(S_{t+1})|S_t=s\right],
\qquad
m_2^b(s)
=
\mathbb{E}_b\!\left[\phi(S_{t+2})|S_t=s\right].
\]
These quantities summarize the future feature directions observed under the
logging process as functions of the current state. The critic feature at
state $s$ is then
\begin{equation}
\label{eq:pop_critic}
z(s)
=
\begin{pmatrix}
\phi(s) &
\gamma\,m_1^b(s) &
\gamma^2\,m_2^b(s)
\end{pmatrix}^\top
\in\mathcal{H}^3.
\end{equation}

We use conditional feature means under the behavior distribution because the critic is constructed from the transition information available in the observed trajectories. This gives a simple critic that is consistent with the observed data and depends only on the current state. Using behavior-conditioned predictions also avoids introducing additional importance weighting into critic construction. Under substantial policy mismatch, however, these predictive directions need not coincide with those most relevant under the target-policy dynamics.

Other choices are also possible. For example, one could reweight the
conditional expectations to construct features under the target
policy or another distribution. The required weighting would depend on the
prediction horizon, with different adjustments for the one-step and two-step
terms. We therefore use the simpler behavior-policy construction throughout.
Target-policy correction enters separately through the importance ratios in
the Bellman residual introduced in Section~\ref{sec:estimation}.

Define the critic class
\[
\mathcal{C}
=
\left\{
h_\beta(\cdot)
=
\langle\beta,z(\cdot)\rangle_{\mathcal{H}^3}
:
\beta\in\mathcal{H}^3
\right\}.
\]
Each critic $h_\beta$ is scalar-valued since it is given by an inner product
in $\mathcal{H}^3$. Writing
$\beta=(\beta_0,\beta_1,\beta_2)$ with $\beta_j\in\mathcal{H}$, the critic
can assign separate weights to the current-state, one-step predictive, and
two-step predictive components.

Identification with this restricted critic class depends on its richness
relative to the chosen value-function class. In particular, the critic class
must be rich enough to distinguish nonzero Bellman residuals generated by
the value functions under consideration. Similar richness or completeness
conditions are standard in minimax and conditional-moment methods
\citep{uehara2020minimax,dikkala2020minimax,chen2022well}.

\begin{assumption}[Bounded feature map]
\label{ass:bounded_feature}
The feature map $\phi$ is measurable and uniformly bounded in norm. That is,
there exists a constant $B_\phi>0$ such that $
\sup_{s\in\mathcal{S}}
\|\phi(s)\|_{\mathcal{H}}
\leq B_\phi.
$
\end{assumption}

Assumption~\ref{ass:bounded_feature} guarantees that the conditional feature
means in \eqref{eq:pop_critic} are well defined and that $z$ is bounded.
The restricted critic enforces the moment condition
$
\mathbb{E}_{S\sim \mu_b}
[z(S)\delta_V^\alpha(S)]=0,
$
thereby testing the residual along both current-state and predictive feature
directions. 

These richness conditions are important because Bellman-residual criteria
need not be well aligned with value error in general
\citep{fujimoto2022itrustyoubellman}. Mixed-AC does not remove this issue
automatically: with a restricted critic class, small tested moments need not
imply small value error without sufficient richness. The adaptive critic is
intended to make the tested directions more informative, rather than to
provide a universal equivalence between Bellman residual and value error.

The adaptive critic should be viewed as a data-dependent choice of
test-function class rather than a unique characterization of the Bellman
residual. In general, identification depends on whether the chosen critic
space is sufficiently rich for the value-function class under consideration.
The conditional-feature representation adopted here provides one practical
construction that adapts to the transition structure observed in the data, while alternative choices, such as kernel conditional mean embeddings
or model-based transition predictors, could also be incorporated within the
same minimax framework.

When $\phi$ is an RKHS feature map, these are conditional mean embeddings
\citep{song}. Rather than solving kernel Gram-matrix systems which can be computationally expensive
for large datasets, we estimate them
by supervised neural regression on an independent critic-construction sample.

Specifically, let $\mathcal{D}_{\mathrm{crit}}$ denote the
critic-construction sample. We fit two predictive models
$
\widehat m_1, \widehat m_2:\mathcal{S}\to\mathcal{H},
$
using $\mathcal{D}_{\mathrm{crit}}$, where $\widehat m_1$ predicts
$\phi(S')$ from $\phi(S)$ and $\widehat m_2$ predicts $\phi(S'')$ from $\phi(S)$. These
models can be fit using any suitable regression method. In our implementation,
we use neural networks whose input is the current feature vector $\phi(S)$ for faster computations. The fitted predictors are
then frozen, and the value function is estimated on an independent fitting
sample.

\paragraph{Sample Splitting.} To ensure that the empirical critic remains a
function of the current state only, the data are divided into two disjoint
subsets. A critic-construction sample
$\mathcal D_n^{\mathrm{crit}}$
is used to estimate
$\widehat m_1$
and
$\widehat m_2$.
These predictors are then frozen. The value function is subsequently
estimated on an independent fitting sample
$\mathcal D_n^{\mathrm{fit}}$,
where the critic feature
$\widehat z(S)$
is evaluated only through the current state. Consequently, conditional on
$\mathcal D_n^{\mathrm{crit}}$,
the learned critic is fixed throughout the value-estimation step.

For an observation $i$ in the independent value-fitting sample, define the
estimated critic feature
$
\widehat Z_i
=
\widehat z(S_i)
=
\begin{pmatrix}
\phi(S_i) &
\gamma\,\widehat m_1(S_i) &
\gamma^2\,\widehat m_2(S_i)
\end{pmatrix}^\top
\in\mathcal{H}^3.$
Thus, $\widehat Z_i$ depends on the fitting observation only through its
current state $S_i$. The future-state observations used to train
$\widehat m_1$ and $\widehat m_2$ come from the separate
critic-construction sample.

Let
$
r_\alpha(V)
=
\widehat{\delta_V^\alpha}(\mathbf{S})
=
\left(
\widehat{\delta_V^\alpha}(S_1),
\dots,
\widehat{\delta_V^\alpha}(S_n)
\right)^\top
\in\mathbb{R}^n
$
denote the empirical mixed residual vector, and let
$\widehat Z:\mathcal{H}^3\to\mathbb{R}^n$ be the empirical critic operator
defined by
$(\widehat Z\beta)_i
=
\langle\widehat Z_i,\beta\rangle_{\mathcal{H}^3}.
$
The finite-sample minimax problem becomes
\begin{equation}
\label{eq:minimax_adaptive}
\min_{V\in\mathcal{V}}
\left[
\max_{\beta\in\mathcal{H}^3}
\left\{
\frac{1}{n}
\left\langle
\beta,
\widehat Z^*r_\alpha(V)
\right\rangle_{\mathcal{H}^3}
-
\lambda_h
\|\beta\|_{\mathcal{H}^3}^2
\right\}
+
\lambda_V\Omega(V)
\right].
\end{equation}
where
$\widehat Z^*:\mathbb{R}^n\to\mathcal{H}^3$ is the adjoint operator
$
\widehat Z^*y
=
\sum_{i=1}^n y_i\widehat Z_i,
$ and $\lambda_h>0$ is a critic regularization parameter.

For fixed $V$, the inner maximization has a closed-form solution, leading to
the objective
\begin{equation}
\label{eq:adaptive_objective}
\widehat J(V)
=
\frac{1}{4\lambda_h n^2}
r_\alpha(V)^\top
\widehat{\mathbf K}_n
r_\alpha(V)
+
\lambda_V\Omega(V),
\end{equation}
where $\Omega(V)$ is a regularization term on the value-function class. The
derivation of \eqref{eq:adaptive_objective} is given in
Appendix~\ref{app:closed_form_inner}.

The matrix
$
\widehat{\mathbf K}_n
=
\widehat Z\widehat Z^*
$
is the kernel matrix induced by the learned critic representation, with
$
(\widehat{\mathbf K}_n)_{ij}
=
\langle
\widehat z(S_i),
\widehat z(S_j)
\rangle_{\mathcal{H}^3}.
$ It is a data-dependent kernel on the current state that incorporates learned
one-step and two-step predictive transition features. Thus, unlike a fixed
kernel such as
$
k_\sigma(s,\widetilde s)
=
\exp\!\left(
-\frac{\|s-\widetilde s\|^2}{2\sigma^2}
\right),
$
which measures similarity using the current states alone, the proposed kernel
also compares how those current states predict future feature directions
under the observed dynamics.

With finite-dimensional features, we evaluate the empirical moment directly
without forming the $n\times n$ Gram matrix; for fixed feature dimension this
is linear in the number of fitting tuples.

\subsection{Estimation}\label{sec:estimation}

We now describe how the proposed framework can be implemented
using finite samples. Recall that the population objective
depends on the mixed Bellman residual. Since $\mathbf{P}$ is
unknown, we estimate the residuals from trajectories generated
by the behavior policy $\pi_b$.

\begin{assumption}[Policy overlap]
\label{ass:behavior_support}
For every state $s$, the target policy $\pi(\cdot|s)$ is absolutely
continuous with respect to the behavior policy $\pi_b(\cdot|s)$. Equivalently,
in the discrete-action case,
$
\pi_b(a|s)>0
$
whenever
$
\pi(a|s)>0.
$
\end{assumption}

Assumption~\ref{ass:behavior_support} ensures that the importance ratio
$\pi(a|s)/\pi_b(a|s)$ is well defined on the support of the target policy.
For continuous action spaces, the ratio is understood in terms of densities
with respect to a common reference measure.

We use per-decision importance sampling. For a fitting tuple
$O_i=(S_i,A_i,R_i,S_i',A_i',R_i',S_i'')$, define
\[
\rho_{0,i}=\frac{\pi(A_i|S_i)}{\pi_b(A_i|S_i)},
\qquad
\rho_{1,i}=\frac{\pi(A_i'|S_i')}{\pi_b(A_i'|S_i')},
\qquad
\rho_{01,i}=\rho_{0,i}\rho_{1,i}.
\]
The one-step and two-step residual estimators are
\begin{align}
\widehat{\delta_V^{(1)}}(S_i)
&=
\rho_{0,i}\Big(R_i+\gamma V(S_i')\Big)-V(S_i),\\
\widehat{\delta_V^{(2)}}(S_i)
&=
\rho_{0,i}R_i
+
\gamma\rho_{01,i}R_i'
+
\gamma^2\rho_{01,i}V(S_i'')
-
V(S_i).
\end{align}
The empirical mixed residual is
$
\widehat{\delta_V^{\alpha}}(\mathbf S)
=
(1-\alpha)\widehat{\delta_V^{(1)}}(\mathbf S)
+
\alpha\widehat{\delta_V^{(2)}}(\mathbf S).
$

Let $r_{\alpha,V}(O_i)$ denote the resulting importance-weighted mixed
residual for the observed tuple $O_i$. Under exact importance ratios,
\[
\mathbb{E}_{\pi_b}
\left[
r_{\alpha,V}(O_i)|S_i=s
\right]
=
\delta_V^\alpha(s).
\]
Thus, the empirical residual is conditionally unbiased for the population
target-policy mixed Bellman residual. The critic representation introduced in
Section~\ref{sec:adaptive_critic}
is first learned on the critic-construction sample and then kept fixed.
The empirical Bellman residual is evaluated on the independent fitting
sample. This separation prevents the value-fitting stage from reusing
the data employed to construct the adaptive critic.

\paragraph{Finite-Sample Stability.}
The following result analyzes the fitting objective under independent fitting
tuples and exact, unnormalized importance ratios.
When $\phi:\mathcal{S}\to\mathbb{R}^d$ is finite-dimensional,
we obtain a finite-sample stability result conditional on the
critic-construction sample. Let
$
O_i
=
(S_i,A_i,R_i,S_i',A_i',R_i',S_i'')
$
denote one fitting tuple, and consider the linear value class
$V_\theta(s)=\theta^\top\phi(s)$. Define
$
\widehat m(\theta)
=
\frac{1}{n_t}
\sum_{i=1}^{n_t}
\widehat z(S_i)r_\alpha(\theta;O_i),
$
and
$
m_{\mathrm{crit}}(\theta)
=
\mathbb{E}
\left[
\widehat z(S)r_\alpha(\theta;O)
\mid
\mathcal{D}_n^{\mathrm{crit}}
\right].
$
Multiplying \eqref{eq:adaptive_objective} by $4\lambda_h$ does not change its
minimizer. Set $\lambda=4\lambda_h\lambda_V$ and consider the equivalent
objective below. Let
\[
\widehat J(\theta)
=
\|\widehat m(\theta)\|_2^2
+
\lambda\|\theta\|_2^2,
\qquad
J_{\mathrm{crit}}(\theta)
=
\|m_{\mathrm{crit}}(\theta)\|_2^2
+
\lambda\|\theta\|_2^2.
\]

For the finite-sample result below, we also assume that, conditional on
$\mathcal{D}_n^{\mathrm{crit}}$, the fitted conditional-mean predictors are uniformly bounded.

\begin{assumption}[Bounded estimated critic]
\label{ass:bounded_estimated_critic}
Conditional on $\mathcal{D}_n^{\mathrm{crit}}$, there exist finite
constants $B_{m,1}$ and $B_{m,2}$ such that
\[
\sup_{s\in\mathcal{S}}
\|\widehat m_1(s)\|_2
\le B_{m,1}, \quad
\sup_{s\in\mathcal{S}}
\|\widehat m_2(s)\|_2
\le B_{m,2}.\]
\end{assumption}

We next give a finite-sample result for the estimator obtained after the
adaptive critic has been learned on the independent critic-construction
sample. The result controls the empirical objective relative to its
population counterpart conditional on the fitted critic.

\begin{proposition}[Finite-Sample Stability]
\label{prop:finite_sample_stability}
Consider the fixed linear class
$V_\theta(s)=\theta^\top\phi(s)$, with
$\Theta=\{\theta:\|\theta\|_2\le R_\theta\}$.
Suppose Assumptions~\ref{ass:bounded_reward}--\ref{ass:bounded_estimated_critic} hold, and assume additionally that
$\rho_0\le\bar\rho$ and $\rho_{01}\le\bar\rho_2$ almost surely. Define
$\ell_{d,\delta}=\log(8d/\delta)$.
If
$
\widehat\theta
\in
\argmin_{\theta\in\Theta}
\widehat J(\theta),
$
then, conditional on $\mathcal{D}_n^{\mathrm{crit}}$, there exists a
universal constant $C>0$ such that, with probability at least
$1-\delta$ over the fitting sample,
\[
J_{\mathrm{crit}}(\widehat\theta)
\le
\inf_{\theta\in\Theta}
J_{\mathrm{crit}}(\theta)
+
CB_Z^2B_r^2
\left\{
\sqrt{\frac{\ell_{d,\delta}}{n_t}}
+
\frac{\ell_{d,\delta}}{n_t}
\right\},
\]
where
$
B_Z
=
\left(
B_\phi^2
+
\gamma^2B_{m,1}^2
+
\gamma^4B_{m,2}^2
\right)^{1/2},
$
and
$
B_r
=
\bar\rho R_{\max}
+
\alpha\gamma\bar\rho_2R_{\max}
+
R_\theta
[
(1-\alpha)(1+\gamma\bar\rho)B_\phi$ $
+
\alpha(1+\gamma^2\bar\rho_2)B_\phi
].
$
\end{proposition}

Proposition~\ref{prop:finite_sample_stability} gives an objective-level
finite-sample guarantee for the second-stage estimator conditional on the
learned critic. It assumes independent fitting tuples and exact bounded ratios,
whereas the experiments also use trajectories, estimated ratios, and
self-normalization. An end-to-end value-error bound would additionally require
restricted-class identification and control of first-stage critic estimation.

\section{Experiments}\label{sec:experiments}

We evaluate \textbf{Mixed-AC} (mixed residual with adaptive critic) on two simulated environments and four continuous-control benchmarks from MetaWorld~\citep{mclean2025metaworld}. In all experiments, the discount factor is $\gamma=0.9$. We evaluate
$\widehat V^\pi(s)$ on a fixed set of test states and report results over
$100$ independent repetitions. Hyperparameters are selected using a small
number of independent pilot replications and are then held fixed in the main experiment. Full implementation details and additional experiments are provided in Appendix~\ref{app:experiments}.

\paragraph{Environments.}
\textbf{Env I} is a 5-dimensional simulated continuous-state MDP with binary actions. The
initial state is Gaussian, the transition is linear with Gaussian noise, and
the reward is the negative squared distance from a fixed target point. Both
the behavior and target policies are logistic, with the behavior policy
estimated using only the first three state coordinates to introduce mild
misspecification. \textbf{Env II} is a partially observed environment in which the
true system evolves in an 8-dimensional hidden state but the agent observes
only a noisy 5-dimensional projection, so a single observation does not fully
determine the hidden state. We use this setting only as an empirical robustness
stress test under violation of the observed-state Markov assumption; it is not
covered by the MDP identification theory. \textbf{Env III--VI} are four robotic manipulation tasks from
MetaWorld~\citep{mclean2025metaworld}: \texttt{door-open-v3},
\texttt{drawer-open-v3}, \texttt{button-press-v3}, and
\texttt{faucet-open-v3}. Each task has a 39-dimensional state and a
4-dimensional continuous action. The behavior policy is a mixture of the
target policy and a broader exploration policy, creating mismatch between the
behavior and target distributions. Reference values are approximated by
target-policy Monte Carlo rollout; for Env II, each evaluation observation is
paired with a fixed canonical latent completion as detailed in
Appendix~\ref{app:experiments}.

\begin{figure}[H]
\centering
\includegraphics[width=\linewidth]{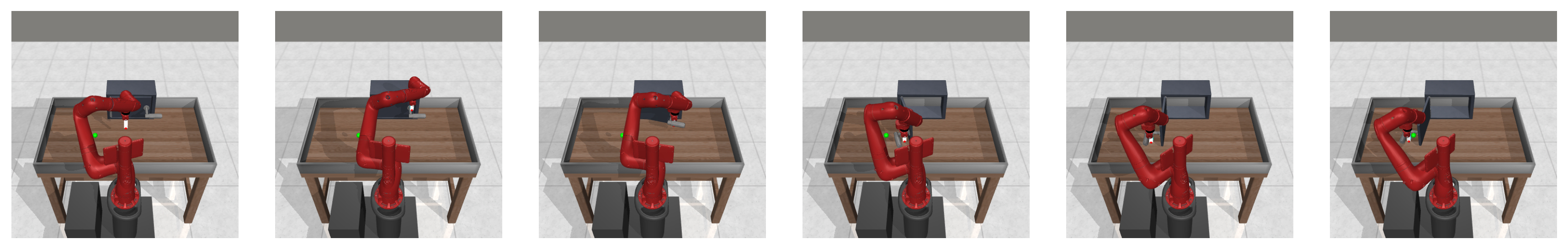}
\caption{Frames from the MetaWorld \texttt{door-open-v3} task. The robot arm
approaches the handle, makes contact, and pulls the door open.}
\label{fig:metaworld_frames}
\end{figure}

\paragraph{Implementation.}
We use the value class $V_\theta(s)=\theta^\top\phi(s)$, where $\phi$
combines random Fourier features and fixed MLP features. The conditional
means $\widehat m_1(S)$ and $\widehat m_2(S)$ are estimated using neural
networks on a separate set of trajectories and then kept fixed during
value-function fitting. In MetaWorld, the target-policy density is known
exactly, while the behavior density is estimated from the logged data. We
use self-normalized importance ratios because multi-step weights can be
highly variable. The corresponding sensitivity analysis is reported in
Appendix~\ref{app:ratio_sensitivity}. We evaluate
$\alpha\in\{0,0.25,0.5,0.75,1\}$.

\paragraph{Bias--Variance Behavior.}
We use Env III to illustrate one setting in which residual mixing produces a bias--variance tradeoff. As shown in
Figure~\ref{fig:door_bias_variance}, increasing $\alpha$ reduces squared
bias while increasing variance. The resulting mean MSE has a shallow interior
minimum at $\alpha=0.5$. In particular, squared bias decreases from
$0.04601$ to $0.03824$, while variance increases from $0.03798$ to
$0.04578$. The corresponding MSE decreases from $0.08361$ at $\alpha=0$
to $0.08180$ at $\alpha=0.5$, before increasing to $0.08356$ at
$\alpha=1$. This pattern need not hold equally strongly in every setting,
since the effect of $\alpha$ also depends on the value class, transition
dynamics, and finite-sample estimation.

\begin{figure}[t]
\centering
\includegraphics[width=0.75\linewidth]{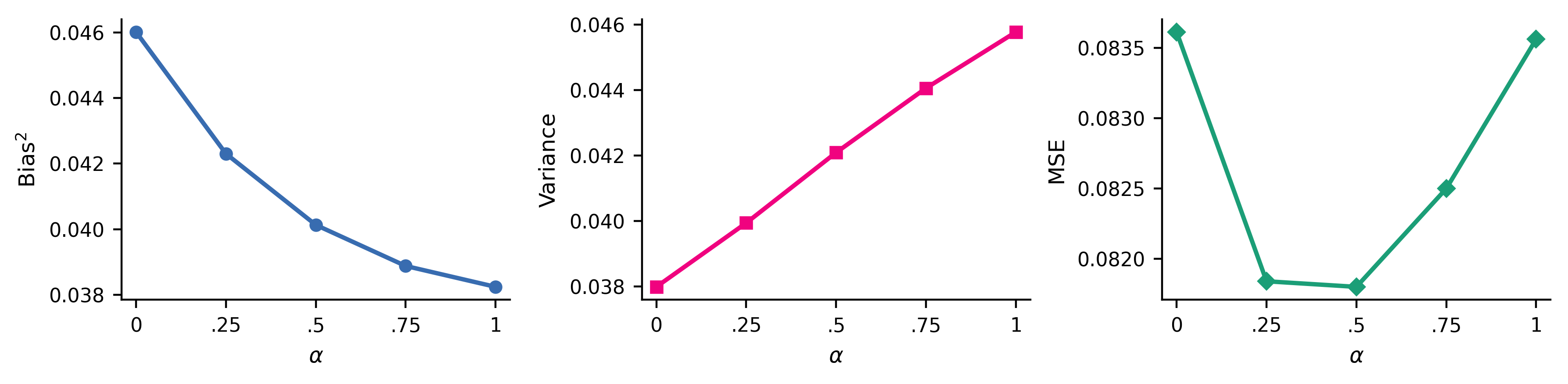}
\caption{Pointwise squared bias, variance, and MSE in the MetaWorld
\texttt{door-open-v3} task as functions of $\alpha$. The two-step
contribution reduces squared bias but increases variance, with a shallow
MSE minimum at $\alpha=0.5$.}
\label{fig:door_bias_variance}
\end{figure}

\paragraph{Comparison.}
Table~\ref{tab:comparison} compares pointwise value estimation with Fitted Q
Evaluation (FQE)~\citep{le2019batch} and Kernel BRM/MQL
~\citep{kernelloss,uehara2020minimax}. We report Mixed-AC at $\alpha=0$
and the best fixed $\alpha$, where the latter is chosen post hoc using true
MSE and is shown only as a diagnostic. Selecting $\alpha$ using only offline
data is a separate model-selection problem; the observable validation
criteria examined in Appendix~\ref{app:alpha_selection} do not uniformly
recover the MSE-minimizing candidate. We also report scalar policy-value error for
$J^\pi=\mathbb{E}_{S\sim\nu}[V^\pi(S)]$ against
DualDICE~\citep{nachum2019dualdice}, GenDICE~\citep{zhang2020gendice},
and BestDICE~\citep{yang2020offpolicy}. These distribution-correction
methods are evaluated only for the scalar policy value. For this scalar
comparison, Mixed-AC uses the critic-scaled validation rule described in
Appendix~\ref{app:alpha_selection}.

\begin{table}[t]
\centering
\caption{Pointwise MSE and scalar policy-value squared error, reported as
mean $\pm$ SE over $100$ repetitions. The best fixed $\alpha$ is chosen
post hoc using true MSE and is included only as a diagnostic.}
\label{tab:comparison}
\small
\setlength{\tabcolsep}{2.3pt}
\resizebox{\linewidth}{!}{%
\begin{tabular}{lcccccc}
\toprule
\textbf{Method}
& \textbf{Env I} & \textbf{Env II}
& \textbf{Env III} & \textbf{Env IV}
& \textbf{Env V} & \textbf{Env VI} \\
\midrule
\multicolumn{7}{l}{\textbf{Pointwise Value MSE}}\\
FQE
& $\mathbf{.01777{\scriptstyle \pm .00089}}$
& $.02217{\scriptstyle \pm .00033}$
& $.69464{\scriptstyle \pm .01138}$
& $.08911{\scriptstyle \pm .00460}$
& $.01105{\scriptstyle \pm .00029}$
& $1.95669{\scriptstyle \pm .02731}$ \\

Kernel BRM/MQL
& $.11884{\scriptstyle \pm .00225}$
& $.03378{\scriptstyle \pm .00023}$
& $1.22100{\scriptstyle \pm .10387}$
& $2.62950{\scriptstyle \pm .26174}$
& $.13920{\scriptstyle \pm .01169}$
& $8.17740{\scriptstyle \pm .63216}$ \\

Mixed-AC ($\alpha=0$)
& $.14404{\scriptstyle \pm .00612}$
& $.01850{\scriptstyle \pm .00039}$
& $.08361{\scriptstyle \pm .00819}$
& $\mathbf{.05861{\scriptstyle \pm .00814}}$
& $.003393{\scriptstyle \pm .000634}$
& $.14376{\scriptstyle \pm .01713}$ \\

Mixed-AC (best fixed $\alpha$)
& $.14404{\scriptstyle \pm .00612}$
& $\mathbf{.01828{\scriptstyle \pm .00035}}$
& $\mathbf{.08180{\scriptstyle \pm .00808}}$
& $\mathbf{.05861{\scriptstyle \pm .00814}}$
& $\mathbf{.003186{\scriptstyle \pm .000397}}$
& $\mathbf{.11888{\scriptstyle \pm .01035}}$ \\
\midrule
\multicolumn{7}{l}{\textbf{$J^\pi$ Squared Error}}\\
Mixed-AC
& $.04725{\scriptstyle \pm .00775}$
& $\mathbf{.00095{\scriptstyle \pm .00013}}$
& $\mathbf{.11495{\scriptstyle \pm .01546}}$
& $\mathbf{.07719{\scriptstyle \pm .00867}}$
& $\mathbf{.00352{\scriptstyle \pm .000711}}$
& $\mathbf{.15939{\scriptstyle \pm .02217}}$ \\

DualDICE
& $.11974{\scriptstyle \pm .02682}$
& $.12567{\scriptstyle \pm .00758}$
& $.64101{\scriptstyle \pm .00824}$
& $.16485{\scriptstyle \pm .00402}$
& $.01224{\scriptstyle \pm .00020}$
& $.82602{\scriptstyle \pm .01283}$ \\

GenDICE
& $1.78283{\scriptstyle \pm .16016}$
& $.00252{\scriptstyle \pm .00002}$
& $.12537{\scriptstyle \pm .01305}$
& $.08068{\scriptstyle \pm .00941}$
& $.00864{\scriptstyle \pm .00074}$
& $.32441{\scriptstyle \pm .03335}$ \\

BestDICE
& $\mathbf{.00394{\scriptstyle \pm .00047}}$
& $.007350{\scriptstyle \pm .000410}$
& $.43554{\scriptstyle \pm .01625}$
& $.10980{\scriptstyle \pm .00479}$
& $.01527{\scriptstyle \pm .00045}$
& $.69331{\scriptstyle \pm .01481}$ \\
\bottomrule
\end{tabular}%
}
\end{table}

For pointwise estimation, FQE and Kernel BRM/MQL perform better in Env I,
a simple fully observed Markov setting, whereas Mixed-AC has lower MSE than
both baselines in Env II and all four MetaWorld tasks. For scalar evaluation, Mixed-AC has the lowest mean squared error in five of
six environments, while BestDICE is lowest in Env I. Paired intervals favor
Mixed-AC over BestDICE in Env II through Env VI. Mixed-AC also has lower mean error than GenDICE in all six environments,
although the paired intervals include zero in Env III and Env IV, and has
lower error than DualDICE in all six environments.
Additional fixed-$\alpha$ results, practical offline selection procedures,
critic ablations, importance-ratio sensitivity experiments, paired
comparisons, and implementation details are provided in
Appendix~\ref{app:experiments}.

\section{Conclusion}

We proposed Mixed-AC, which combines a mixed Bellman residual with an
adaptive critic based on conditional predictions of future features. Across the experiments, the choice of $\alpha$ changes the effect of
one-step and two-step information, while the critic adapts its test functions
to the observed transition structure. In some settings, this mixing also
produces a useful bias--variance tradeoff. The experiments suggest that the method is particularly useful in more complex settings, such as partially observed, higher-order, or nonlinear environments, where using a data-dependent kernel can improve value estimation. 

However, performance depends on tuning choices, including the value and
critic function classes and the choice of $\alpha$. Our validation
experiments also show that selecting $\alpha$ using only offline data can be
difficult, making more reliable offline tuning an important direction for
future work.

\bibliographystyle{asa}
\bibliography{causal.bib}

\appendix

\section{Proof of Proposition \ref{contraction}}\label{contractionproof}

Let $V_1,V_2:\mathcal{S}\to\mathbb{R}$ be any two bounded functions. By definition,
\[
\T_\pi^{\alpha}V_1 - \T_\pi^{\alpha}V_2
=
(1-\alpha)\big(\T_\pi V_1 - \T_\pi V_2\big)
+
\alpha\big(\T_\pi^{(2)}V_1 - \T_\pi^{(2)}V_2\big).
\]
Taking the supremum norm and applying the triangle inequality,
\begin{align*}
\|\T_\pi^{\alpha}V_1 - \T_\pi^{\alpha}V_2\|_\infty
&\le
(1-\alpha)\|\T_\pi V_1 - \T_\pi V_2\|_\infty
+
\alpha\|\T_\pi^{(2)}V_1 - \T_\pi^{(2)}V_2\|_\infty.
\end{align*}
Since $\T_\pi$ is a $\gamma$-contraction under the supremum norm, we know $
\|\T_\pi V_1 - \T_\pi V_2\|_\infty
\le
\gamma\|V_1 - V_2\|_\infty
$. Moreover, $\T_\pi^{(2)}=\T_\pi\circ\T_\pi$ is a $\gamma^2$-contraction, i.e., $
\|\T_\pi^{(2)}V_1 - \T_\pi^{(2)}V_2\|_\infty
\le
\gamma^2\|V_1 - V_2\|_\infty
$. Substituting these bounds gives
\begin{align*}
\|\T_\pi^{\alpha}V_1 - \T_\pi^{\alpha}V_2\|_\infty
&\le
(1-\alpha)\gamma\|V_1 - V_2\|_\infty
+
\alpha\gamma^2\|V_1 - V_2\|_\infty \\
&=
\big((1-\alpha)\gamma + \alpha\gamma^2\big)
\|V_1 - V_2\|_\infty.
\end{align*}
This proves the contraction inequality. Finally, if $\alpha>0$ and $0<\gamma<1$, then
$$
(1-\alpha)\gamma + \alpha\gamma^2
=
\gamma\big(1-\alpha + \alpha\gamma\big)
<
\gamma
$$
since $\gamma<1$ implies $\alpha\gamma<\alpha$. Therefore, $\T_\pi^{\alpha}$ is a contraction with a strictly smaller Lipschitz constant than $\T_\pi$ whenever $\alpha >0$. 

\section{Proof of Theorem \ref{thm:identification}}
\label{ident_proof}

\begin{proof}
Let \(g(s)=\delta_V^\alpha(s)\). Since
\(\delta_V^\alpha\in L^1(\mu_b)\),
\(\mathcal{L}(V,h)\) is well-defined for every bounded measurable critic
function \(h\).

First suppose that \(g(s)=0\) \(\mu_b\)-almost everywhere. Then, for
any bounded measurable \(h\),
\[
\mathcal{L}(V,h)
=
\int h(s)g(s)\,\mu_b(ds)
=
0.
\]
This proves \((ii)\Rightarrow(i)\).

Conversely, suppose that \(\mathcal{L}(V,h)=0\) for every bounded
measurable critic function \(h\). Then
\[
\int h(s)g(s)\,\mu_b(ds)=0
\]
for every bounded measurable \(h\). Taking
\(h(s)=\mathbf{1}\{g(s)>0\}\) gives
\[
0
=
\int_{\{g>0\}} g(s)\,\mu_b(ds).
\]
Since \(g(s)>0\) on \(\{g>0\}\), it follows that
\(\mu_b(\{g>0\})=0\). Similarly, taking
\(h(s)=\mathbf{1}\{g(s)<0\}\) gives
\(\mu_b(\{g<0\})=0\). Hence
\[
\mu_b(\{g\neq 0\})=0,
\]
which means that
\[
\delta_V^\alpha(s)=0
\quad
\mu_b\text{-almost everywhere}.
\]
This proves \((i)\Rightarrow(ii)\).

Now suppose \(V=V^\pi\). Since \(V^\pi\) is the fixed point of the
one-step Bellman operator,
\[
\mathcal{T}_\pi V^\pi=V^\pi.
\]
Therefore
\[
\mathcal{T}_\pi^{(2)}V^\pi
=
\mathcal{T}_\pi(\mathcal{T}_\pi V^\pi)
=
V^\pi.
\]
Thus
\[
\delta_{V^\pi}^{(1)}(s)=0,
\qquad
\delta_{V^\pi}^{(2)}(s)=0
\]
for all \(s\in\mathcal{S}\), and hence
\[
\delta_{V^\pi}^{\alpha}(s)=0
\]
for all \(s\in\mathcal{S}\). By the equivalence above,
\[
\mathcal{L}(V^\pi,h)=0
\]
for every bounded measurable critic function \(h\).

Finally, suppose that $\mathcal{L}(V,h)=0$ for every bounded measurable
critic function $h$. From the equivalence above,
\[
\delta_V^\alpha(s)=0
\qquad
\mu_b\text{-almost everywhere}.
\]
If $\delta_V^\alpha$ is continuous and $\mu_b$ has full support on
$\mathcal S$, this implies
\[
\delta_V^\alpha(s)=0
\qquad
\text{for all }s\in\mathcal S.
\]
Therefore $\mathcal{T}_\pi^\alpha V=V$. By
Proposition~\ref{contraction}, $\mathcal{T}_\pi^\alpha$ has a unique bounded
fixed point. Since $V^\pi$ is also a fixed point, it follows that
$V=V^\pi$.

This completes the proof.
\end{proof}

\section{Derivation of the Closed-Form Inner Optimization}
\label{app:closed_form_inner}

In this section, we derive the closed-form solution of the inner maximization problem. Recall that the empirical minimax objective is
\begin{equation}
\min_{V\in\mathcal{V}}
\left[
\max_{\beta\in\mathcal{H}^3}
\left\{
\frac{1}{n} \langle \beta, \widehat Z^* r_{\alpha}(V) \rangle_{\mathcal{H}^3}
-
\lambda_h \|\beta\|_{\mathcal{H}^3}^2
\right\}
+
\lambda_V\Omega(V)
\right].
\end{equation}
where $r_{\alpha}(V)\in\mathbb{R}^n$ is the empirical mixed residual vector and $\widehat Z^*: \mathbb{R}^n \to \mathcal{H}^3$ is the adjoint operator defined by $\widehat Z^* y = \sum_{i=1}^n y_i \widehat Z_i$.

For fixed $V$, define 
\begin{equation}
q(V) = \frac{1}{n} \widehat Z^* r_{\alpha}(V).
\end{equation}
Then the inner problem can be written as
\begin{equation}
\max_{\beta\in\mathcal{H}^3}
\left\{
\langle \beta, q(V) \rangle_{\mathcal{H}^3} - \lambda_h \langle \beta, \beta \rangle_{\mathcal{H}^3}
\right\}.
\end{equation}
This is a concave quadratic functional of $\beta$. Taking the Fréchet derivative with respect to $\beta$ and setting it to zero gives
\begin{equation}
q(V) - 2\lambda_h \beta = 0.
\end{equation}
Solving for $\beta$, we obtain
\begin{equation}
\beta^\star_{\alpha}(V) = \frac{1}{2\lambda_h} q(V) = \frac{1}{2\lambda_h n} \widehat Z^* r_{\alpha}(V).
\end{equation}

Substituting this back into the objective gives
\begin{align}
\langle \beta^\star_{\alpha}(V), q(V) \rangle_{\mathcal{H}^3} - \lambda_h \|\beta^\star_{\alpha}(V)\|_{\mathcal{H}^3}^2
&= \frac{1}{2\lambda_h} \langle q(V), q(V) \rangle_{\mathcal{H}^3} - \lambda_h \frac{1}{4\lambda_h^2} \langle q(V), q(V) \rangle_{\mathcal{H}^3} \\
&= \frac{1}{4\lambda_h} \langle q(V), q(V) \rangle_{\mathcal{H}^3}.
\end{align}
Using the definition of $q(V)$ and the property of the adjoint $\langle \widehat Z^* y, \widehat Z^* y \rangle_{\mathcal{H}^3} = y^\top \widehat Z\widehat Z^* y$, we have
\begin{equation}
\langle q(V), q(V) \rangle_{\mathcal{H}^3} = \frac{1}{n^2} r_{\alpha}(V)^\top \widehat Z\widehat Z^* r_{\alpha}(V) = \frac{1}{n^2} r_{\alpha}(V)^\top \widehat{\mathbf K}_n r_{\alpha}(V),
\end{equation}
where $\widehat{\mathbf K}_n$ is the $n \times n$ kernel matrix with $(\widehat{\mathbf K}_n)_{ij} = \langle \widehat Z_i, \widehat Z_j \rangle_{\mathcal{H}^3}$. Therefore the outer objective becomes
\begin{equation}
\widehat{J}(V) = \frac{1}{4\lambda_h n^2} r_{\alpha}(V)^\top \widehat{\mathbf K}_n r_{\alpha}(V) + \lambda_V\,\Omega(V).
\end{equation}

\section{Proof of Proposition~\ref{prop:finite_sample_stability}}
\label{app:stability}

We prove the finite-sample result conditional on the
critic-construction sample $\mathcal{D}_n^{\mathrm{crit}}$.
For notational simplicity, write $n$ for the fitting sample
size $n_t$ and set $\ell_{d,\delta}=\log(8d/\delta)$. Conditional on
$\mathcal{D}_n^{\mathrm{crit}}$,
the fitted predictors $\widehat{m}_1$ and $\widehat{m}_2$
are fixed functions.

Let $\phi:\mathcal{S}\to\mathbb{R}^d$ be a fixed feature map
and consider the linear value class
$V_\theta(s)=\theta^\top\phi(s)$ with
$\Theta=\{\theta\in\mathbb{R}^d:\|\theta\|_2\le R_\theta\}$.
For a fitting observation
$O=(S,A,R,S',A',R',S'')$,
define the critic feature
\[
\widehat{z}(S)
=
\begin{pmatrix}
\phi(S)\\[2pt]
\gamma\,\widehat{m}_1(S)\\[2pt]
\gamma^2\widehat{m}_2(S)
\end{pmatrix}
\in\mathbb{R}^{3d}.
\]
The feature $\widehat{z}(S)$ depends on the fitting observation only through
$S$.

Let $\rho_0=\pi(A|S)/\pi_b(A|S)$,
$\rho_1=\pi(A'|S')/\pi_b(A'|S')$, and $\rho_{01}=\rho_0\rho_1$.
For $V_\theta$, define
\begin{align*}
\widehat\delta^{(1)}_\theta(O)
&=
\rho_0\{R+\gamma V_\theta(S')\}-V_\theta(S),\\[4pt]
\widehat\delta^{(2)}_\theta(O)
&=
\rho_0 R
+\gamma\rho_{01}R'
+\gamma^2\rho_{01}V_\theta(S'')-V_\theta(S),
\end{align*}
and the mixed residual
$r_\alpha(\theta;O)
=(1-\alpha)\widehat\delta^{(1)}_\theta(O)
+\alpha\widehat\delta^{(2)}_\theta(O)$.

For independent fitting observations $O_1,\ldots,O_n$, define
\[
\widehat{m}(\theta)
=\frac{1}{n}\sum_{i=1}^n\widehat{z}(S_i)r_\alpha(\theta;O_i),
\qquad
m_{\mathrm{crit}}(\theta)
=\mathbb{E}\!\left[
\widehat{z}(S)r_\alpha(\theta;O)
\mid\mathcal{D}_n^{\mathrm{crit}}
\right],
\]
and
$\widehat{J}(\theta)=\|\widehat{m}(\theta)\|_2^2+\lambda\|\theta\|_2^2$,
$J_{\mathrm{crit}}(\theta)
=\|m_{\mathrm{crit}}(\theta)\|_2^2+\lambda\|\theta\|_2^2$.

\begin{proof}
Throughout, all expectations and probabilities are conditional
on $\mathcal{D}_n^{\mathrm{crit}}$, under which
$\widehat{m}_1$, $\widehat{m}_2$, and $\widehat{z}$ are fixed
while the fitting observations remain independent.

Since $V_\theta(s)=\theta^\top\phi(s)$, direct expansion gives
\begin{align*}
\widehat\delta^{(1)}_\theta(O)
&=\rho_0 R+\{\gamma\rho_0\phi(S')-\phi(S)\}^\top\theta,\\[4pt]
\widehat\delta^{(2)}_\theta(O)
&=\rho_0 R+\gamma\rho_{01}R'
+\{\gamma^2\rho_{01}\phi(S'')-\phi(S)\}^\top\theta.
\end{align*}
Therefore the mixed residual is affine in $\theta$,
\[
r_\alpha(\theta;O)=b_\alpha(O)+a_\alpha(O)^\top\theta,
\]
where
$b_\alpha(O)=\rho_0 R+\alpha\gamma\rho_{01}R'$
and
\[
a_\alpha(O)
=(1-\alpha)\{\gamma\rho_0\phi(S')-\phi(S)\}
+\alpha\{\gamma^2\rho_{01}\phi(S'')-\phi(S)\}.
\]

By Assumptions~\ref{ass:bounded_reward}
and~\ref{ass:bounded_feature},
$\|\phi(s)\|_2\le B_\phi$, $|R|\le R_{\max}$, $|R'|\le R_{\max}$.
Combined with $\rho_0\le\bar\rho$ and $\rho_{01}\le\bar\rho_2$
almost surely,
\[
|b_\alpha(O)|
\le\bar\rho R_{\max}+\alpha\gamma\bar\rho_2 R_{\max}
=:B_b,
\]
and
\[
\|a_\alpha(O)\|_2
\le(1-\alpha)(1+\gamma\bar\rho)B_\phi
+\alpha(1+\gamma^2\bar\rho_2)B_\phi
=:B_a.
\]
Hence for every $\theta\in\Theta$,
$|r_\alpha(\theta;O)|\le B_b+B_a R_\theta=:B_r$.

By Assumption~\ref{ass:bounded_estimated_critic},
$\|\widehat{m}_1(s)\|_2\le B_{m,1}$
and $\|\widehat{m}_2(s)\|_2\le B_{m,2}$, so
\[
\|\widehat{z}(S)\|_2^2
\le B_\phi^2+\gamma^2 B_{m,1}^2+\gamma^4 B_{m,2}^2
=:B_Z^2.
\]
It follows that
$\|\widehat{z}(S)r_\alpha(\theta;O)\|_2\le B_Z B_r$
uniformly over $\theta\in\Theta$.

Define
\[
\widehat{q}
=\frac{1}{n}\sum_{i=1}^n\widehat{z}(S_i)b_\alpha(O_i),
\qquad
q_{\mathrm{crit}}
=\mathbb{E}\!\left[
\widehat{z}(S)b_\alpha(O)\mid\mathcal{D}_n^{\mathrm{crit}}
\right],
\]
and
\[
\widehat{M}
=\frac{1}{n}\sum_{i=1}^n\widehat{z}(S_i)a_\alpha(O_i)^\top,
\qquad
M_{\mathrm{crit}}
=\mathbb{E}\!\left[
\widehat{z}(S)a_\alpha(O)^\top\mid\mathcal{D}_n^{\mathrm{crit}}
\right].
\]
Then $\widehat{m}(\theta)=\widehat{q}+\widehat{M}\theta$ and
$m_{\mathrm{crit}}(\theta)=q_{\mathrm{crit}}+M_{\mathrm{crit}}\theta$,
so
\[
\widehat{m}(\theta)-m_{\mathrm{crit}}(\theta)
=(\widehat{q}-q_{\mathrm{crit}})
+(\widehat{M}-M_{\mathrm{crit}})\theta.
\]
Taking norms and using $\|\theta\|_2\le R_\theta$,
\[
\sup_{\theta\in\Theta}
\|\widehat{m}(\theta)-m_{\mathrm{crit}}(\theta)\|_2
\le
\|\widehat{q}-q_{\mathrm{crit}}\|_2
+R_\theta\|\widehat{M}-M_{\mathrm{crit}}\|_{\mathrm{op}}.
\]

For the vector term, define the centered $3d\times1$ matrices
\[
X_i^{(q)}=\widehat z(S_i)b_\alpha(O_i)-
\mathbb E[\widehat z(S)b_\alpha(O)\mid\mathcal D_n^{\mathrm{crit}}].
\]
The uncentered vector has norm at most $B_ZB_b$, so
$\|X_i^{(q)}\|_2\le 2B_ZB_b$ and a crude variance proxy is
$v_q\le 4nB_Z^2B_b^2$. The rectangular matrix Bernstein inequality gives,
with probability at least $1-\delta/2$,
\[
\|\widehat{q}-q_{\mathrm{crit}}\|_2
\le C_1 B_Z B_b
\left\{
\sqrt{\frac{\ell_{d,\delta}}{n}}
+\frac{\ell_{d,\delta}}{n}
\right\}.
\]
For the matrix term, define the centered $3d\times d$ matrices
\[
X_i^{(M)}=\widehat z(S_i)a_\alpha(O_i)^\top-
\mathbb E[\widehat z(S)a_\alpha(O)^\top\mid\mathcal D_n^{\mathrm{crit}}].
\]
Here $\|X_i^{(M)}\|_{\mathrm{op}}\le 2B_ZB_a$ and a crude variance proxy is
$v_M\le 4nB_Z^2B_a^2$. Matrix Bernstein gives, with probability at least
$1-\delta/2$,
\[
\|\widehat{M}-M_{\mathrm{crit}}\|_{\mathrm{op}}
\le C_2 B_Z B_a
\left\{
\sqrt{\frac{\ell_{d,\delta}}{n}}
+\frac{\ell_{d,\delta}}{n}
\right\}.
\]
These are standard bounds for sums of bounded random vectors
and matrices; see, e.g.,
\citet{tropp2015introduction}.
A union bound gives, with probability at least $1-\delta$,
\[
\sup_{\theta\in\Theta}
\|\widehat{m}(\theta)-m_{\mathrm{crit}}(\theta)\|_2
\le C_3 B_Z B_r
\left\{
\sqrt{\frac{\ell_{d,\delta}}{n}}
+\frac{\ell_{d,\delta}}{n}
\right\},
\]
where we used $B_r=B_b+B_a R_\theta$.

Since $\|m_{\mathrm{crit}}(\theta)\|_2\le B_Z B_r$ and
$\|\widehat{m}(\theta)\|_2\le B_Z B_r$, the triangle
inequality gives
\[
|\widehat{J}(\theta)-J_{\mathrm{crit}}(\theta)|
=\bigl|\|\widehat{m}(\theta)\|_2^2
-\|m_{\mathrm{crit}}(\theta)\|_2^2\bigr|
\le 2B_Z B_r
\|\widehat{m}(\theta)-m_{\mathrm{crit}}(\theta)\|_2.
\]
Hence, with probability at least $1-\delta$,
\[
\sup_{\theta\in\Theta}
|\widehat{J}(\theta)-J_{\mathrm{crit}}(\theta)|
\le C_4 B_Z^2 B_r^2
\left\{
\sqrt{\frac{\ell_{d,\delta}}{n}}
+\frac{\ell_{d,\delta}}{n}
\right\}.
\]

Let $\widehat\theta\in\argmin_{\theta\in\Theta}\widehat{J}(\theta)$.
For any $\theta\in\Theta$,
\begin{align*}
J_{\mathrm{crit}}(\widehat\theta)
&\le\widehat{J}(\widehat\theta)
+\sup_{\vartheta\in\Theta}
|\widehat{J}(\vartheta)-J_{\mathrm{crit}}(\vartheta)|\\
&\le\widehat{J}(\theta)
+\sup_{\vartheta\in\Theta}
|\widehat{J}(\vartheta)-J_{\mathrm{crit}}(\vartheta)|\\
&\le J_{\mathrm{crit}}(\theta)
+2\sup_{\vartheta\in\Theta}
|\widehat{J}(\vartheta)-J_{\mathrm{crit}}(\vartheta)|.
\end{align*}
Taking the infimum over $\theta\in\Theta$ and substituting
the uniform bound above, there exists a universal
constant $C>0$ such that
\[
J_{\mathrm{crit}}(\widehat\theta)
\le\inf_{\theta\in\Theta}J_{\mathrm{crit}}(\theta)
+CB_Z^2 B_r^2
\left\{
\sqrt{\frac{\ell_{d,\delta}}{n}}
+\frac{\ell_{d,\delta}}{n}
\right\}.
\]
This completes the proof.
\end{proof}

\section{Experimental Details}
\label{app:experiments}

This appendix gives the full experimental setup and additional results for
Section~\ref{sec:experiments}.

\subsection{Env I Simulated MDP}

The state space is $\mathbb{R}^5$ and the action space is binary,
$\mathcal{A}=\{0,1\}$. The initial state is drawn from
$S_0\sim\mathcal{N}(0,I_5)$.

The transition dynamics are
\[
S_{t+1}=A_{a_t}S_t+\varepsilon_t,
\qquad
\varepsilon_t\sim\mathcal{N}(0,0.05^2I_5),
\]
where
\[
A_0 =
\begin{pmatrix}
0.90 & 0.10 & 0.00 & 0.00 & 0.00 \\
0.00 & 0.82 & 0.10 & 0.00 & 0.00 \\
0.00 & 0.00 & 0.85 & 0.08 & 0.00 \\
0.00 & 0.00 & 0.00 & 0.80 & 0.10 \\
0.00 & 0.00 & 0.00 & 0.00 & 0.78
\end{pmatrix},
\qquad
A_1 =
\begin{pmatrix}
0.72 & 0.20 & 0.08 & 0.00 & 0.00 \\
0.00 & 0.84 & 0.12 & 0.04 & 0.00 \\
0.00 & 0.00 & 0.80 & 0.12 & 0.05 \\
0.00 & 0.00 & 0.00 & 0.82 & 0.12 \\
0.00 & 0.00 & 0.00 & 0.00 & 0.76
\end{pmatrix}.
\]

The reward is
\[
R(s,a)=-\|s-s^\star\|^2,
\qquad
s^\star=(0.1,0.05,0,-0.05,0.08)^\top.
\]
The reward is highest near $s^\star$. Although it does not depend directly
on the action, the action changes the future state and therefore affects
future rewards.

Both the behavior and target policies are logistic functions of the state,
\[
\pi_b(1|s)=\sigma(w_b^\top s),
\qquad
\pi(1|s)=\sigma(w_\pi^\top s),
\]
with
\[
w_b=(-0.3,0.7,-0.5,0.25,-0.15)^\top,
\qquad
w_\pi=(0.45,1.10,-0.10,0.35,0.20)^\top.
\]
The probabilities are clipped to $[0.05,0.95]$.

The behavior policy is treated as unknown and is estimated by logistic
regression using the first three state coordinates. This introduces mild
misspecification relative to the true behavior policy.

Each episode has horizon $60$ and provides $58$ usable two-step tuples. The
experiment targets approximately $2000$ tuples. Since complete
episodes are preserved, this gives $35$ episodes and $2030$ tuples. The
critic-construction sample contains $18$ episodes and $1044$ tuples, while
the value-fitting sample contains $17$ episodes and $986$ tuples.

For evaluation, we use $300$ fixed states drawn from
$\mathcal{N}(0,0.15^2I_5)$. For each state, the true target-policy value is
approximated using $300$ independent Monte Carlo rollouts with horizon $60$
and $\gamma=0.9$. The resulting evaluation states and truth values are kept
fixed across all $100$ repetitions and all methods.

\subsection{Env II Partially Observable MDP}

The system evolves in a latent state $h_t\in\mathbb{R}^8$ according to
\[
h_{t+1}=A^h_{a_t}h_t+\varepsilon_t,
\qquad
\varepsilon_t\sim\mathcal{N}(0,0.05^2I_8).
\]
The matrices $A^h_0$ and $A^h_1$ are fixed stable matrices generated once
at import time from a NumPy generator with seed $999$: a Gaussian matrix is
rescaled to spectral radius $0.87$ for action $0$ and $0.83$ for action $1$.
The observation matrix $C$ is drawn from the same generator and each row is
normalized to unit Euclidean norm. These matrices are stored as deterministic
module constants rather than regenerated for each trajectory. Training trajectories are initialized independently with
$h_0\sim\mathcal{N}(0,I_8)$.

The observed state is
\[
S_t=Ch_t+\eta_t,
\qquad
\eta_t\sim\mathcal{N}(0,0.15^2I_5),
\]
where $C\in\mathbb{R}^{5\times8}$. Thus, the agent observes a noisy
5-dimensional projection of the 8-dimensional latent state.

The reward is
\[
R(S_t,A_t)
=
-0.2\|S_t-s^\star\|^2-0.02A_t,
\qquad
s^\star=(0.1,0.05,0,-0.05,0.08)^\top.
\]

The behavior and target policies use the same logistic weight vectors as
Env I and depend on the observed state. As in Env I, the behavior policy is
estimated using only the first three observed coordinates.

Each episode has horizon $60$. The experiment uses $35$ complete episodes
and $2030$ two-step tuples. The critic-construction sample contains $18$
episodes and $1044$ tuples, while the value-fitting sample contains $17$
episodes and $986$ tuples.

Evaluation uses $300$ fixed observations
$s\sim\mathcal{N}(0,0.15^2I_5)$. For each observation, it deterministically
defines
\[
h^\dagger(s)=\operatorname{clip}(C^\dagger s,-4,4),
\]
where $C^\dagger$ is the Moore--Penrose pseudoinverse. All $300$
target-policy Monte Carlo rollouts for that observation start from
$(S_0,h_0)=(s,h^\dagger(s))$. The simulator-specific reference value is
\[
V_{\mathrm{can}}^\pi(s)
=
\mathbb{E}^{\pi}\!\left[
\sum_{t=0}^{59}\gamma^tR_t
\;\middle|\;
S_0=s,\,
h_0=h^\dagger(s)
\right].
\]
It is approximated with $300$ Monte Carlo rollouts and does not integrate over
a conditional distribution of latent states given $S_0=s$. It should not be
interpreted as an observed-state MDP value. Env II remains only an empirical
partial-observability robustness experiment.

\subsection{MetaWorld environments}
\label{app:metaworld}

We use four tasks from MetaWorld~\citep{mclean2025metaworld}.
Env III is \texttt{door-open-v3}, Env IV is
\texttt{drawer-open-v3}, Env V is \texttt{button-press-v3}, and Env VI is
\texttt{faucet-open-v3}. Each task has a 39-dimensional state and a
4-dimensional continuous action.

The target policy is a squashed Gaussian policy with a fixed state-dependent
pre-squash mean and standard deviation $\sigma_\pi=0.30$. The behavior
policy is a mixture of the target policy and a broader exploration policy,
\[
b(a|s)
=
(1-\varepsilon)\pi(a|s)
+
\varepsilon u(a|s),
\]
where $u$ has noise level $\sigma_u=1.0$. We use $\varepsilon=0.8$ for
Door and $\varepsilon=0.6$ for Drawer, Button, and Faucet.

The primary MetaWorld experiments use $85$ complete episodes and $5015$
tuples. The critic-construction set contains $43$ episodes and $2537$
tuples, while the value-fitting set contains $42$ episodes and $2478$
tuples.

The target-policy density is known analytically, while the behavior action
density is estimated from the logged behavior data using a neural network.
The behavior model is a two-component mixture of squashed diagonal Gaussians.
Its state-dependent means, log standard deviations, and mixture probability
are produced by a two-hidden-layer ReLU network of width $128$; the log
standard deviations are clipped to $[-4,2]$. It is trained by minimizing the
negative squashed-Gaussian mixture log likelihood with Adam (learning rate
$10^{-3}$), for $30$ epochs and minibatches of size $256$, using the logged
state-action pairs. The target density uses the known analytic squashed
Gaussian. Both density calculations use the inverse-tanh boundary clamp
$\varepsilon=10^{-6}$ and the corresponding Jacobian correction. No target
density network is used for the ratios.
The target policy density is evaluated analytically at the logged actions; no
auxiliary target-policy rollout is used for ratio estimation.
Writing $\mu_\pi(s)$ for the known pre-squash mean, the target density is
\[
\pi(a|s)
=
\mathcal N\!\left(
\operatorname{arctanh}(a);
\mu_\pi(s),
\sigma_\pi^2 I_4
\right)
\prod_{j=1}^{4}(1-a_j^2)^{-1}.
\]
The inverse transformation is evaluated with a small numerical stabilization
near the action boundaries. We use
\[
\widehat\rho_0
=
\frac{\pi(A_t|S_t)}
{\widehat b(A_t|S_t)},
\qquad
\widehat\rho_1
=
\frac{\pi(A_{t+1}|S_{t+1})}
{\widehat b(A_{t+1}|S_{t+1})}.
\]

Each MetaWorld task uses $400$ fixed evaluation states obtained from
environment resets. The true value at each state is approximated using $50$
target-policy Monte Carlo rollouts with horizon $60$ and $\gamma=0.9$.

\subsection{Pointwise evaluation}
\label{app:metrics}

For Env II, $V^\pi(s_j)$ in the formulas below denotes the
simulator-specific reference value $V_{\mathrm{can}}^\pi(s_j)$ defined above.

Let $\{s_j\}_{j=1}^M$ denote the fixed evaluation states and let
$\widehat V_r^\pi(s_j)$ be the value estimate in repetition $r$. We use
$R=100$ independent repetitions and define
\[
\overline V^\pi(s_j)
=
\frac{1}{R}\sum_{r=1}^R\widehat V_r^\pi(s_j).
\]

The empirical squared bias is
\[
\mathrm{Bias}^2
=
\frac{1}{M}\sum_{j=1}^M
\left\{
\overline V^\pi(s_j)-V^\pi(s_j)
\right\}^2.
\]

The empirical variance is
\[
\mathrm{Var}
=
\frac{1}{M}\sum_{j=1}^M
\frac{1}{R-1}\sum_{r=1}^R
\left\{
\widehat V_r^\pi(s_j)-\overline V^\pi(s_j)
\right\}^2.
\]

The empirical MSE is
\[
\mathrm{MSE}
=
\frac{1}{RM}\sum_{r=1}^R\sum_{j=1}^M
\left\{
\widehat V_r^\pi(s_j)-V^\pi(s_j)
\right\}^2.
\]
With this variance convention,
\[
\mathrm{MSE}
=
\mathrm{Bias}^2+\frac{R-1}{R}\mathrm{Var}.
\]

For the main method comparisons, we report the mean repetition-level MSE and
its standard error. Paired comparisons use the same generated dataset for
the two methods and report Student-$t$ 95\% confidence intervals.

\subsection{Feature representation}

We use the value class $V_\theta(s)=\theta^\top\phi(s)$, where
\[
\phi(s)
=
\left[
\phi_{\mathrm{RFF}}(s),
\eta\phi_{\mathrm{MLP}}(s)
\right].
\]

The first block consists of random Fourier features for an RBF kernel. Its
coordinates are
\[
[\phi_{\mathrm{RFF}}(s)]_j
=
\sqrt{\frac{2}{D}}\cos(\omega_j^\top s+b_j).
\]
The frequencies are sampled from a centered Gaussian distribution and the
phases are sampled uniformly on $[0,2\pi]$. The RBF bandwidth is chosen by
the median heuristic on the training states.

The second block is obtained from a fixed neural network with $\tanh$
activation. Its weights are generated once and then kept fixed. The scalar
$\eta$ controls the relative size of this block.

Only the coefficient vector $\theta$ is optimized when fitting the value
function.

\subsection{Adaptive critic}

The critic uses conditional means of future features. On
$\mathcal{D}_{\mathrm{crit}}$, we estimate
\[
m_1(s)=\mathbb{E}_b\{\phi(S_{t+1})|S_t=s\},
\qquad
m_2(s)=\mathbb{E}_b\{\phi(S_{t+2})|S_t=s\}
\]
using a neural network.

The network takes $\phi(S_t)$ as input and predicts both
$\phi(S_{t+1})$ and $\phi(S_{t+2})$. After training, the estimated functions
$\widehat m_1$ and $\widehat m_2$ are kept fixed.

The critic used for value fitting is
\[
Z(S)
=
\left[
\phi(S),
\gamma\widehat m_1(S),
\gamma^2\widehat m_2(S)
\right].
\]

The conditional-mean network is trained for $100$ epochs with learning rate
$0.01$. Its hidden width and weight decay are selected separately for each
environment during pilot tuning.

\subsection{Residual construction and value fitting}

The experiments use self-normalized importance ratios
\[
\widetilde\rho_{0,i}
=
\frac{\rho_{0,i}}
{\frac{1}{n}\sum_{j=1}^n\rho_{0,j}},
\qquad
\widetilde\rho_{01,i}
=
\frac{\rho_{0,i}\rho_{1,i}}
{\frac{1}{n}\sum_{j=1}^n\rho_{0,j}\rho_{1,j}}.
\]

The one-step residual is
\[
r_{1,i}(\theta)
=
\widetilde\rho_{0,i}
\left\{
R_i+\gamma V_\theta(S_i')
\right\}
-
V_\theta(S_i),
\]
and the two-step residual is
\[
r_{2,i}(\theta)
=
\widetilde\rho_{0,i}R_i
+
\gamma\widetilde\rho_{01,i}R_i'
+
\gamma^2\widetilde\rho_{01,i}V_\theta(S_i'')
-
V_\theta(S_i).
\]

We combine them as
\[
r_{\alpha,i}(\theta)
=
(1-\alpha)r_{1,i}(\theta)
+
\alpha r_{2,i}(\theta),
\qquad
\alpha\in\{0,0.25,0.5,0.75,1\}.
\]

For a fixed $\alpha$, we estimate $\theta$ by minimizing
\[
\left\|
\frac{1}{n}\sum_{i=1}^n
Z(S_i)r_{\alpha,i}(\theta)
\right\|_2^2
+
\lambda_\theta\|\theta\|_2^2.
\]

We use AdamW with learning rate $0.005$ for $1000$ iterations. The explicit
ridge term is the only regularization applied to $\theta$.

\subsection{Sample splitting}
\label{app:sample_splitting}

All splits are made using complete trajectories.

For the standard experiment, the data are divided into
$\mathcal{D}_{\mathrm{crit}}$ and $\mathcal{D}_{\mathrm{fit}}$. The first
sample is used to estimate the conditional means and construct the critic.
The second sample is used to estimate the value function.

For Env I and Env II,
$\mathcal{D}_{\mathrm{crit}}$ contains $18$ episodes and $1044$ tuples, while
$\mathcal{D}_{\mathrm{fit}}$ contains $17$ episodes and $986$ tuples.

For each MetaWorld environment,
$\mathcal{D}_{\mathrm{crit}}$ contains $43$ episodes and $2537$ tuples, while
$\mathcal{D}_{\mathrm{fit}}$ contains $42$ episodes and $2478$ tuples.

For practical selection of $\alpha$, a third set
$\mathcal{D}_{\mathrm{val}}$ is used. In Env I and Env II, the $35$ episodes
are divided into folds containing $12$, $12$, and $11$ episodes, corresponding
to $696$, $696$, and $638$ tuples. In each MetaWorld environment, the $85$
episodes are divided into folds containing $29$, $28$, and $28$ episodes,
corresponding to $1711$, $1652$, and $1652$ tuples.

The critic-scaled experiment uses separate critic-construction,
value-fitting, and validation subsets. For the common-studentized experiment,
we additionally average the validation score over all six assignments of the
three folds to these roles. All splits are trajectory-disjoint.

\subsection{Pilot hyperparameter selection}
\label{app:pilot}

Hyperparameters are selected using five independent pilot replications before
the main experiment. For each candidate configuration, we compute its average
pointwise MSE on the fixed evaluation states. The configuration with the
smallest pilot MSE is then held fixed for all $100$ repetitions.

For Mixed-AC, pilot tuning is performed at $\alpha=0$. The pilot grid uses
$d_{\mathrm{RFF}}\in\{64,128\}$, $d_{\mathrm{MLP}}\in\{8,16,32\}$,
$\eta\in\{0,0.1,0.25,0.5,1,2,5,10,50,100\}$,
$\lambda_\theta\in\{10^{-5},10^{-4},10^{-3},10^{-2},10^{-1},1,10\}$,
conditional-network hidden width in $\{32,64\}$, and conditional-network
weight decay in $\{10^{-5},10^{-4},10^{-3}\}$. The selected configuration
is then used for every value of $\alpha$.
Keeping these hyperparameters fixed isolates the effect of changing
$\alpha$ within a common estimator configuration; the fixed-$\alpha$ curves
are not intended to represent separately retuned optima at each value of
$\alpha$.

The tuning grids are method-specific because the estimators have different
parameterizations. The comparisons therefore reflect the stated
implementations rather than an equal-count hyperparameter search budget.

The selected configurations are shown in
Table~\ref{tab:mixedac_configs}.

\begin{table}[h]
\centering
\caption{Selected Mixed-AC configurations.}
\label{tab:mixedac_configs}
\small
\setlength{\tabcolsep}{4pt}
\begin{tabular}{lccccc}
\toprule
\textbf{Env.}
& $d_{\mathrm{RFF}}$
& $d_{\mathrm{MLP}}$
& $\eta$
& $\lambda_\theta$
& $(h_{\mathrm{cond}},\lambda_{\mathrm{cond}})$\\
\midrule
I   & 64  & 32 & 5   & 10        & $(32,10^{-4})$\\
II  & 128 & 16 & 1   & $10^{-5}$ & $(32,10^{-5})$\\
III & 64  & 16 & 100 & 10        & $(64,10^{-5})$\\
IV  & 64  & 16 & 10  & $10^{-1}$ & $(32,10^{-5})$\\
V   & 64  & 16 & 10  & 1         & $(32,10^{-3})$\\
VI  & 128 & 32 & 10  & 1         & $(32,10^{-4})$\\
\bottomrule
\end{tabular}
\end{table}

\subsection{Fitted Q Evaluation}

We compare with Fitted Q Evaluation
(FQE)~\citep{le2019batch}. Its hyperparameters are selected independently
using the same five-pilot procedure.

FQE is evaluated in its standard one-step form; the effect of residual
horizon within Mixed-AC is examined separately through the fixed-$\alpha$
and critic-interaction analyses. FQE is fit on the value-fitting split. For
discrete actions its target-policy expectation is summed exactly; for the
continuous MetaWorld actions it is approximated by Monte Carlo with $64$
target-policy action samples per next state in the main FQE configuration.

The FQE grid contains hidden width in $\{64,128\}$, learning rate
$10^{-3}$, weight decay in $\{10^{-5},10^{-4},10^{-3}\}$, the number of
FQE iterations in $\{50,100\}$, and regression steps per iteration in
$\{10,25\}$. This gives $24$ candidate configurations.

The selected configurations are given in
Table~\ref{tab:fqe_configs}.

\begin{table}[h]
\centering
\caption{Selected FQE configurations.}
\label{tab:fqe_configs}
\small
\setlength{\tabcolsep}{4pt}
\begin{tabular}{lccccc}
\toprule
\textbf{Env.} & \textbf{Hidden}
& \textbf{LR} & \textbf{WD}
& \textbf{Iterations} & \textbf{Steps}\\
\midrule
I   & 128 & .001 & .001       & 100 & 25\\
II  & 128 & .001 & $10^{-5}$  & 50  & 25\\
III & 64  & .001 & .001       & 100 & 10\\
IV  & 128 & .001 & .001       & 100 & 25\\
V   & 128 & .001 & $10^{-4}$  & 50  & 25\\
VI  & 128 & .001 & $10^{-4}$  & 50  & 25\\
\bottomrule
\end{tabular}
\end{table}

\subsection{Kernel BRM and MQL}

The kernel baseline follows~\citet{kernelloss} and
\citet{uehara2020minimax}. The value function is
$V_\beta(s)=\sum_{i=1}^n\beta_i k(s_i,s)$, where $k$ is an RBF kernel whose
bandwidth is selected using the median heuristic.

Let $K_n$ be the corresponding Gram matrix. The coefficient vector minimizes
\[
\frac{1}{n^2}r(\beta)^\top K_n r(\beta)
+
\lambda_{\mathrm{RKHS}}\beta^\top K_n\beta.
\]

Using the same five-pilot procedure described above, the regularization
parameter is selected from
$\{10^{-8},10^{-7},10^{-6},10^{-5},10^{-4},10^{-3},
10^{-2},10^{-1},1\}$.

The selected values are
$
\begin{array}{lc}
\toprule
\textbf{Environment} & \lambda_{\mathrm{RKHS}}\\
\midrule
\text{Env I}   & 10^{-5}\\
\text{Env II}  & 10^{-8}\\
\text{Env III} & 10^{-4}\\
\text{Env IV}  & 10^{-4}\\
\text{Env V}   & 10^{-4}\\
\text{Env VI}  & 10^{-4}\\
\bottomrule
\end{array}
$

The regularized quadratic problem is solved using normal equations.
The kernel baseline uses the stated fixed RBF kernel on the original state
coordinates, whereas the hybrid feature representation is part of the
Mixed-AC implementation.

\subsection{Paired pointwise comparisons}

Since the competing estimators are evaluated on the same generated datasets,
we also report paired MSE differences. Negative values favor Mixed-AC.

\begin{table}[h]
\centering
\caption{Paired pointwise MSE differences between Mixed-AC at $\alpha=0$
and FQE. Brackets give 95\% confidence intervals. Negative values favor
Mixed-AC.}
\label{tab:paired_main}
\small
\setlength{\tabcolsep}{7pt}
\begin{tabular}{lc}
\toprule
\textbf{Env.} & \textbf{Mixed-AC $-$ FQE}\\
\midrule
I   & $.12627$ $[.11442,.13812]$\\
II  & $-.00367$ $[-.00454,-.00281]$\\
III & $-.61103$ $[-.64061,-.58145]$\\
IV  & $-.03050$ $[-.05042,-.01057]$\\
V   & $-.007656$ $[-.009002,-.006311]$\\
VI  & $-1.81293$ $[-1.88040,-1.74545]$\\
\bottomrule
\end{tabular}
\end{table}

Mixed-AC at $\alpha=0$ has lower pointwise MSE than FQE in Env II through
Env VI, with the paired confidence intervals excluding zero. In Env I, FQE
performs better. Mixed-AC also has substantially lower mean MSE than Kernel
BRM/MQL in Env II through Env VI, as reported in
Table~\ref{tab:comparison}.

\subsection{Effect of the mixing parameter}

Table~\ref{tab:fixed_alpha_full} gives the full fixed-$\alpha$ experiment.
The same hyperparameter configuration is used at every value of $\alpha$.

\begin{table}[h]
\centering
\caption{Pointwise MSE, mean $\pm$ SE over $100$ repetitions.}
\label{tab:fixed_alpha_full}
\small
\setlength{\tabcolsep}{4pt}
\begin{tabular}{lccccc}
\toprule
& $0$ & $.25$ & $.50$ & $.75$ & $1$\\
\midrule
Env I
& $.14404\pm.00612$
& $.18737\pm.01684$
& $.33502\pm.03250$
& $.55987\pm.04945$
& $.83283\pm.06611$\\
Env II
& $.01850\pm.00039$
& $.01828\pm.00035$
& $.01943\pm.00037$
& $.02166\pm.00043$
& $.02452\pm.00053$\\
Env III
& $.08361\pm.00819$
& $.08184\pm.00799$
& $.08180\pm.00808$
& $.08250\pm.00829$
& $.08356\pm.00853$\\
Env IV
& $.05861\pm.00814$
& $.05977\pm.00829$
& $.06432\pm.00859$
& $.07048\pm.00894$
& $.07708\pm.00928$\\
Env V
& $.003393\pm.000634$
& $.003186\pm.000397$
& $.003693\pm.000423$
& $.004238\pm.000466$
& $.004695\pm.000497$\\
Env VI
& $.14376\pm.01713$
& $.11888\pm.01035$
& $.12987\pm.01041$
& $.14536\pm.01165$
& $.16010\pm.01278$\\
\bottomrule
\end{tabular}
\end{table}

The clearest improvement from a strict mixture occurs in Env VI. The paired
difference between $\alpha=.25$ and $\alpha=0$ is $-0.02489$, with its
95\% confidence interval excluding zero.

Env III has a shallow minimum at $\alpha=.5$. Env V is minimized at $\alpha=.25$ and Env IV favors $\alpha=0$. For Env II, $\alpha=.25$ remains slightly smaller than $\alpha=0$.

\subsection{Bias--variance tradeoff}
\label{app:bias_variance}

We use Env III to examine the bias--variance tradeoff in more detail. The
results are shown in Table~\ref{tab:door_bias_variance}.

\begin{table}[h]
\centering
\caption{Pointwise squared bias, variance, and MSE for
\texttt{door-open-v3}.}
\label{tab:door_bias_variance}
\small
\setlength{\tabcolsep}{6pt}
\begin{tabular}{cccc}
\toprule
$\alpha$ & \textbf{Bias$^2$} & \textbf{Variance} & \textbf{MSE}\\
\midrule
0    & .04601 & .03798 & .08361\\
.25  & .04230 & .03994 & .08184\\
.50  & .04013 & .04209 & .08180\\
.75  & .03888 & .04406 & .08250\\
1    & .03824 & .04578 & .08356\\
\bottomrule
\end{tabular}
\end{table}

As $\alpha$ increases, squared bias decreases from $.04601$ to $.03824$,
while variance increases from $.03798$ to $.04578$. The resulting mean MSE
first decreases and then increases, with a shallow minimum at $\alpha=.5$.

\subsection{Interaction between Residual Mixing and Critic Representation}
\label{app:critic_ablation}

The mixed residual and adaptive critic provide two different ways of
incorporating information from future transitions. Increasing $\alpha$
places more weight on the two-step Bellman residual, while the adaptive
critic augments current-state features with predicted one-step and two-step
future features. We therefore examine how the benefit of the adaptive critic
changes as the contribution of the two-step residual increases.

For each $\alpha$, we compare the adaptive critic with a critic using only
current-state features. We report the relative MSE improvement
\[
G(\alpha)
=
100\times
\frac{
\operatorname{MSE}_{\mathrm{current}}(\alpha)
-
\operatorname{MSE}_{\mathrm{adaptive}}(\alpha)
}{
\operatorname{MSE}_{\mathrm{current}}(\alpha)
}.
\]
Positive values indicate that the adaptive critic improves value estimation.

\begin{table}[h]
\centering
\caption{Relative MSE improvement (\%) from using the adaptive critic instead
of the current-state critic. Positive values favor the adaptive critic.}
\label{tab:critic_residual_interaction}
\small
\setlength{\tabcolsep}{6pt}
\begin{tabular}{lccccc}
\toprule
& $\alpha=0$ & $\alpha=.25$ & $\alpha=.50$ & $\alpha=.75$ & $\alpha=1$\\
\midrule
Env III (Door)
& $-0.5$ & $-0.6$ & $-0.8$ & $-0.9$ & $-1.2$\\
Env IV (Drawer)
& $\mathbf{24.3}$ & $13.4$ & $1.2$ & $-10.3$ & $-19.2$\\
Env V (Button)
& $\mathbf{83.4}$ & $60.6$ & $12.8$ & $-33.9$ & $-53.4$\\
Env VI (Faucet)
& $\mathbf{85.2}$ & $69.1$ & $33.2$ & $-9.3$ & $-36.0$\\
\bottomrule
\end{tabular}
\end{table}

The interaction is especially strong in Drawer, Button, and Faucet. At
$\alpha=0$, the adaptive critic reduces MSE by approximately $24\%$, $83\%$,
and $85\%$, respectively. As $\alpha$ increases, these gains decrease
monotonically. At larger values of $\alpha$, the interaction reverses in
these three tasks, with the current-state critic giving lower MSE. This suggests that the two mechanisms are complementary rather than simply
additive: as more two-step information enters directly through the residual,
the additional benefit of predictive future features in the critic becomes
smaller.

Door behaves differently: the two critic representations are nearly
indistinguishable across the full range of $\alpha$, with relative
differences below $1.2\%$.

Together, these results highlight the complementary roles of the two
components of Mixed-AC. Future transition information can enter through both
the Bellman residual and the critic representation, and their relative
importance depends on the environment and the choice of $\alpha$.

\subsection{Practical Offline Selection of $\alpha$}
\label{app:alpha_selection}

The fixed-$\alpha$ experiments above show that the preferred residual horizon
can vary across environments. In practice, however, the pointwise value error
is unavailable, so selecting $\alpha$ from offline data is a separate
model-selection problem. This difficulty is not specific to Mixed-AC and has
been observed more generally in offline policy evaluation
\citep{tang2021modelselection,liu2025opeselection}. We consider two fully
observable validation criteria. Neither criterion uses the true value function
to select $\alpha$.

\paragraph{Critic-scaled validation.}
The first criterion evaluates the mixed Bellman moment on a held-out
validation sample. The trajectories are divided into
$\mathcal{D}_{\mathrm{crit}}$, $\mathcal{D}_{\mathrm{fit}}$, and
$\mathcal{D}_{\mathrm{val}}$. For each candidate $\alpha$, the critic is
constructed on $\mathcal{D}_{\mathrm{crit}}$, the value function is fitted on
$\mathcal{D}_{\mathrm{fit}}$, and the validation score is computed on
$\mathcal{D}_{\mathrm{val}}$.

For each validation coordinate let
\[
c_j
=
n_{\mathrm{val}}^{-1}\sum_i Z_{ij}^2,
\qquad
c_{\max}=\max_j c_j,
\qquad
\mathcal I_{\mathrm{act}}
=
\{j:c_j>10^{-14}c_{\max}\}.
\]
The same active-coordinate rule is applied to the frozen critic for all
candidate values of $\alpha$. We define
\[
L_{\mathrm{scaled}}(\alpha)
=
\frac{1}{|\mathcal I_{\mathrm{act}}|}
\sum_{j\in\mathcal I_{\mathrm{act}}}
\frac{
\left[
n_{\mathrm{val}}^{-1}
\sum_i Z_{ij}r_{\alpha,i}
\right]^2
}{
n_{\mathrm{val}}^{-1}
\sum_i Z_{ij}^2
}.
\]
The candidate with the smallest score is selected, with ties resolved in
favor of the smaller value of $\alpha$.

Table~\ref{tab:alpha_selection_scaled} reports the resulting pointwise MSE.
The oracle candidate is computed using ground truth only after fitting and is
included as a diagnostic.

\begin{table}[h]
\centering
\caption{Performance of the critic-scaled validation rule. Excess MSE is
the difference between the selected estimator and the oracle candidate in
the same repetition.}
\label{tab:alpha_selection_scaled}
\small
\setlength{\tabcolsep}{4pt}
\begin{tabular}{lcccc}
\toprule
\textbf{Env.}
& \textbf{Selected MSE}
& \textbf{Oracle MSE}
& \textbf{Excess MSE [95\% CI]}
& \textbf{Agreement}\\
\midrule
I
& $.15918$ & $.12798$
& $.03121\ [.01771,.04470]$ & $59\%$\\
II
& $.01899$ & $.01823$
& $.00076\ [.00044,.00109]$ & $53\%$\\
III
& $.10619$ & $.08980$
& $.01639\ [.01142,.02137]$ & $39\%$\\
IV
& $.09221$ & $.07846$
& $.01374\ [.00627,.02122]$ & $65\%$\\
V
& $.004124$ & $.002030$
& $.002094\ [.000608,.003580]$ & $55\%$\\
VI
& $.22188$ & $.11568$
& $.10620\ [.07193,.14047]$ & $41\%$\\
\bottomrule
\end{tabular}
\end{table}

The critic-scaled criterion is relatively conservative in the MetaWorld
experiments and tends to favor smaller values of $\alpha$. This is consistent
with the additional variability introduced by multi-step importance weighting.
At the same time, the fixed-$\alpha$ results show that the pointwise
MSE-minimizing value of $\alpha$ can differ across environments. We therefore
also consider a second criterion that evaluates all candidates using common
one-step and two-step Bellman conditions.

\paragraph{Common-studentized validation.}
The second criterion evaluates every fitted candidate using the same
one-step and two-step Bellman conditions rather than the candidate-specific
mixed residual. For a fitted value function $\widehat V_\alpha$, define on
the validation sample
\[
X_{hij}(\alpha)
=
Z_{ij}r_{h,i}(\widehat V_\alpha),
\qquad
h\in\{1,2\},
\]
where $r_{1,i}$ and $r_{2,i}$ denote the one-step and two-step residuals,
respectively. Let
\[
\widehat m_{hj}(\alpha)
=
\frac{1}{n_{\mathrm{val}}}
\sum_i X_{hij}(\alpha)
\]
and
\[
\widehat s_{hj}^2(\alpha)
=
\frac{1}{n_{\mathrm{val}}-1}
\sum_i
\left\{
X_{hij}(\alpha)-\widehat m_{hj}(\alpha)
\right\}^2.
\]
For each active coordinate, set
\[
\epsilon_j
=
\operatorname{eps}_{\mathrm{mach}}\max\{1,c_j\},
\qquad
\operatorname{eps}_{\mathrm{mach}}
=
2.220446049250313\times10^{-16},
\]
the value returned by \texttt{np.finfo(float).eps} in the implementation.
We use the studentized score
\[
L_{\mathrm{stud}}(\alpha)
=
\frac{1}{2|\mathcal I_{\mathrm{act}}|}
\sum_{h=1}^2
\sum_{j\in\mathcal I_{\mathrm{act}}}
\frac{
n_{\mathrm{val}}\widehat m_{hj}(\alpha)^2
}{
\widehat s_{hj}^2(\alpha)+\epsilon_j
}.
\]
The same active-coordinate threshold is used here.

For this experiment, each logged dataset is divided into three
trajectory-disjoint folds. We consider all six assignments of these folds to
critic construction, value fitting, and validation. For each $\alpha$, the
six validation scores are averaged and the candidate with the smallest
average score is selected. After selection, the candidate is refitted using
the common final two-way split used in the standard experiment. Ground truth
is used only afterward to evaluate pointwise MSE and compute the oracle
candidate.

Table~\ref{tab:alpha_selection_studentized} reports results over $100$
independent repetitions.

\begin{table}[h]
\centering
\caption{Performance of the common-studentized validation rule over $100$
repetitions. The $\alpha=0$ and oracle columns are computed under the same
final-refit protocol. Ground truth is used only for post-selection
evaluation.}
\label{tab:alpha_selection_studentized}
\small
\setlength{\tabcolsep}{4pt}
\begin{tabular}{lccccc}
\toprule
\textbf{Env.}
& \textbf{Selected MSE}
& $\boldsymbol{\alpha=0}$ \textbf{ MSE}
& \textbf{Oracle MSE}
& \textbf{Agreement}
& \textbf{Spearman}\\
\midrule
I   & $.81944$  & $.14799$  & $.11823$  & $3\%$  & $-.736$\\
II  & $.01962$  & $.01819$  & $.01763$  & $9\%$  & $.144$\\
III & $.10686$  & $.10366$  & $.08984$  & $39\%$ & $-.039$\\
IV  & $.05880$  & $.04967$  & $.04223$  & $38\%$ & $.031$\\
V   & $.004233$ & $.002717$ & $.001387$ & $12\%$ & $-.503$\\
VI  & $.15029$  & $.14344$  & $.06845$  & $12\%$ & $-.230$\\
\bottomrule
\end{tabular}
\end{table}

The common-studentized criterion produces a broader range of selected values
of $\alpha$ than the critic-scaled criterion. The comparison also shows that
the ordering induced by an observable Bellman-moment criterion need not agree
with the ordering induced by pointwise value error. In particular, the two
criteria can favor different members of the same fixed-$\alpha$ estimator
family.

Taken together, these experiments illustrate the distinction between
constructing the Mixed-AC estimator family and selecting a particular member
of that family using only offline data. Both criteria are fully observable,
but their behavior depends on the validation objective, as is common in OPE
model selection. We therefore use these experiments to study practical
selection rather than as part of the fixed-$\alpha$ evaluation of residual
mixing. Developing more accurate offline selection criteria is an interesting
direction for future work.

\subsection{Scalar policy-value evaluation}
\label{app:scalar_ope}

We also evaluate the scalar value
\[
J^\pi
=
\mathbb{E}_{S\sim\nu}[V^\pi(S)],
\]
where $\nu$ is the corresponding evaluation-state distribution. For Mixed-AC, $\alpha$ is chosen using the critic-scaled validation procedure
described in Appendix~\ref{app:alpha_selection}.

We compare Mixed-AC with DualDICE~\citep{nachum2019dualdice},
GenDICE~\citep{zhang2020gendice}, and BestDICE~\citep{yang2020offpolicy}.
These methods estimate distribution corrections and are evaluated only for
the scalar policy value, rather than for pointwise $V^\pi(s)$ estimation.
This provides a complementary comparison with distribution-correction OPE
methods.

The complete results are shown in Table~\ref{tab:scalar_full}.

\begin{table}[h]
\centering
\caption{Squared scalar policy-value error over $100$ repetitions.
Entries report mean with SE in parentheses. Bold denotes the lowest
sample mean in each environment.}
\label{tab:scalar_full}
\small
\setlength{\tabcolsep}{3pt}
\renewcommand{\arraystretch}{1.4}
\begin{tabular}{@{}lcccccc@{}}
\toprule
\textbf{Method} & \textbf{I} & \textbf{II} & \textbf{III}
& \textbf{IV} & \textbf{V} & \textbf{VI}\\
\midrule

Mixed-AC
& \shortstack{$.04725$\\{\scriptsize $(.00775)$}}
& \shortstack{$\mathbf{.000945}$\\{\scriptsize $\mathbf{(.000134)}$}}
& \shortstack{$\mathbf{.11495}$\\{\scriptsize $\mathbf{(.01546)}$}}
& \shortstack{$\mathbf{.07719}$\\{\scriptsize $\mathbf{(.00867)}$}}
& \shortstack{$\mathbf{.003524}$\\{\scriptsize $\mathbf{(.000711)}$}}
& \shortstack{$\mathbf{.15939}$\\{\scriptsize $\mathbf{(.02217)}$}}\\

DualDICE
& \shortstack{$.11974$\\{\scriptsize $(.02682)$}}
& \shortstack{$.12567$\\{\scriptsize $(.00758)$}}
& \shortstack{$.64101$\\{\scriptsize $(.00824)$}}
& \shortstack{$.16485$\\{\scriptsize $(.00402)$}}
& \shortstack{$.01224$\\{\scriptsize $(.00020)$}}
& \shortstack{$.82602$\\{\scriptsize $(.01283)$}}\\

GenDICE
& \shortstack{$1.78283$\\{\scriptsize $(.16016)$}}
& \shortstack{$.002522$\\{\scriptsize $(.000198)$}}
& \shortstack{$.12537$\\{\scriptsize $(.01305)$}}
& \shortstack{$.08068$\\{\scriptsize $(.00941)$}}
& \shortstack{$.008640$\\{\scriptsize $(.000738)$}}
& \shortstack{$.32441$\\{\scriptsize $(.03335)$}}\\

BestDICE
& \shortstack{$\mathbf{.003943}$\\{\scriptsize $\mathbf{(.000465)}$}}
& \shortstack{$.007350$\\{\scriptsize $(.000410)$}}
& \shortstack{$.43554$\\{\scriptsize $(.01625)$}}
& \shortstack{$.10980$\\{\scriptsize $(.00479)$}}
& \shortstack{$.015265$\\{\scriptsize $(.000447)$}}
& \shortstack{$.69331$\\{\scriptsize $(.01481)$}}\\

\bottomrule
\end{tabular}
\end{table}

Since the two estimators are evaluated on the same generated datasets, we
also report paired differences in squared scalar error. Negative values favor
Mixed-AC.

\begin{table}[h]
\centering
\caption{Paired differences in squared scalar policy-value error between
Mixed-AC and DualDICE.}
\label{tab:scalar_paired}
\small
\setlength{\tabcolsep}{7pt}
\begin{tabular}{lcc}
\toprule
\textbf{Env.} & \textbf{Mixed-AC $-$ DualDICE} & \textbf{95\% CI}\\
\midrule
I   & $-.07249$  & $[-.12826,-.01673]$\\
II  & $-.12473$  & $[-.13976,-.10969]$\\
III & $-.52605$  & $[-.55486,-.49725]$\\
IV  & $-.08766$  & $[-.10775,-.06757]$\\
V   & $-.008714$ & $[-.01013,-.00730]$\\
VI  & $-.66663$  & $[-.71669,-.61657]$\\
\bottomrule
\end{tabular}
\end{table}

\begin{table}[h]
\centering
\caption{Paired differences in squared scalar policy-value error for
GenDICE and BestDICE. Negative values favor Mixed-AC.}
\label{tab:scalar_paired_newdice}
\small
\setlength{\tabcolsep}{4pt}
\resizebox{\linewidth}{!}{%
\begin{tabular}{lcccc}
\toprule
\textbf{Env.}
& \textbf{Mixed-AC $-$ GenDICE}
& \textbf{95\% CI}
& \textbf{Mixed-AC $-$ BestDICE}
& \textbf{95\% CI}\\
\midrule
I
& $-1.73558$ & $[-2.05330,-1.41785]$
& $.043307$ & $[.028012,.058602]$\\
II
& $-.001577$ & $[-.002045,-.001110]$
& $-.006405$ & $[-.007228,-.005582]$\\
III
& $-.010414$ & $[-.046102,.025274]$
& $-.32059$ & $[-.36459,-.27658]$\\
IV
& $-.003484$ & $[-.026470,.019503]$
& $-.032608$ & $[-.053661,-.011555]$\\
V
& $-.005116$ & $[-.006922,-.003310]$
& $-.011741$ & $[-.013352,-.010130]$\\
VI
& $-.16502$ & $[-.24426,-.08579]$
& $-.53392$ & $[-.58324,-.48461]$\\
\bottomrule
\end{tabular}%
}
\end{table}

Mixed-AC has the lowest mean scalar squared error in Env II through Env VI,
while BestDICE has the lowest mean error in Env I. The paired
Mixed-AC $-$ BestDICE intervals exclude zero in all six environments, favoring
BestDICE in Env I and Mixed-AC in the remaining five environments. Mixed-AC
also has lower mean error than GenDICE in all six environments; the paired
intervals exclude zero in Env I, Env II, Env V, and Env VI, but include zero
in Env III and Env IV. Mixed-AC has lower scalar squared error than DualDICE
in all six environments, with all six paired intervals excluding zero.

The magnitude of the error differs substantially across environments, so the
absolute MSEs should be interpreted relative to the scale of the underlying
policy values.

\subsection{DualDICE}
\label{app:dualdice}

DualDICE~\citep{nachum2019dualdice} estimates the discounted state-action
distribution correction
\[
w_{\pi/\mathcal D}(s,a)
=
\frac{d^\pi(s,a)}{d^{\mathcal D}(s,a)}
\]
without requiring knowledge of the behavior policy or explicit per-step
importance ratios.

To avoid overloading the symbol $\nu$, which denotes our evaluation-state
distribution, let $u(s,a)$ denote the primal DualDICE function and let
$\zeta(s,a)$ denote the dual function. Our implementation uses the general
convex DualDICE objective

\begin{equation*}
\begin{split}
\min_u\max_\zeta \biggl\{ & \frac{ \sum_i w_i \left[ \left\{ u(S_i,A_i) - \gamma \mathbb{E}_{A'\sim\pi(\cdot|S_i')} u(S_i',A') \right\} \zeta(S_i,A_i) - f^*(\zeta(S_i,A_i)) \right] }{ \sum_i w_i } \\
& - (1-\gamma) \mathbb{E}_{\substack{S_0\sim\nu\\A_0\sim\pi(\cdot|S_0)}} u(S_0,A_0) \biggr\}
\end{split}
\end{equation*}

where
\[
w_i=\gamma^{t_i-\min(t)},
\qquad
f(x)=\frac{|x|^{3/2}}{3/2},
\qquad
f^*(y)=\frac{|y|^3}{3}.
\]

At the population saddle point, the dual function $\zeta$ represents the
discounted state-action distribution correction. We use separate neural
networks for $u$ and $\zeta$. Each network has one hidden layer with $64$
$\tanh$ units. Optimization uses Adam for $10{,}000$ steps with minibatch
size $512$, primal learning rate $10^{-4}$, and dual learning rate $10^{-3}$.
The final density-ratio estimate is averaged over the last four saved
$\zeta$ networks.

For Env I and Env II, which have discrete actions, the expectation
\[
\mathbb{E}_{A'\sim\pi(\cdot|S')}
u(S',A')
\]
is computed exactly by summing over the available actions. MetaWorld has a
4-dimensional continuous action space. In these environments, we approximate
the same target-policy expectation by Monte Carlo sampling,
\[
\mathbb{E}_{A'\sim\pi(\cdot|S')}
u(S',A')
\approx
\frac{1}{8}
\sum_{k=1}^8
u(S',A'_k),
\qquad
A'_k\sim\pi(\cdot|S').
\]
The initial-state target-action expectation is approximated in the same way.
Thus, the MetaWorld experiments use the original DualDICE variational
objective with Monte Carlo approximation of the continuous target-action
expectations. The final configs use $64$ action samples for Env I and Env II
and $8$ action samples for each MetaWorld task. MetaWorld state inputs are
standardized using moments computed from the data.

DualDICE uses the complete dataset in each repetition rather than the
critic-construction, value-fitting, and validation partitions used by
Mixed-AC. This gives $2030$ tuples in Env I and Env II and $5015$ tuples in each MetaWorld task.

The policy-value criterion in the original DualDICE formulation is normalized
by $(1-\gamma)$. Since our estimand is the unnormalized discounted value
$J^\pi=\mathbb{E}_{S\sim\nu}[V^\pi(S)]$, we report the self-normalized
estimator
\[
\widehat J^\pi_{\mathrm{DD}}
=
\frac{1}{1-\gamma}
\frac{
\sum_i
\gamma^{t_i}
\widehat\zeta(S_i,A_i)R_i
}{
\sum_i
\gamma^{t_i}
\widehat\zeta(S_i,A_i)
}.
\]
The same evaluation-state distribution $\nu$ used to define $J^\pi$ is used
as the initial-state distribution in the DualDICE objective. Both methods are therefore scored against the same simulator reference scalar.

The neural parameterization of $\zeta$ is unconstrained, so its finite-sample
outputs are not forced to be nonnegative. Across the experiments, some
negative estimated corrections are observed, with the largest amount in
Env III. Nevertheless, all $100$ optimization runs completed successfully in
each environment.

\subsection{GenDICE and BestDICE}
\label{app:newdice}

We additionally compare with GenDICE~\citep{zhang2020gendice} and
BestDICE~\citep{yang2020offpolicy}. Both methods are implemented using the
NeuralDICE parameterization with separate networks for the primal and dual
functions. Each network has two hidden layers of width $64$ with ReLU
activations and Xavier-uniform initialization.

Optimization uses Adam with learning rate $10^{-4}$ for the primal, dual,
and normalization variables, $(\beta_1,\beta_2)=(.9,.999)$, epsilon
$10^{-7}$, elementwise gradient clipping at $1$, minibatch size $2048$, and
$100{,}000$ updates. We use $\gamma=.9$ and the quadratic function
\[
f(x)=\frac{|x|^2}{2}.
\]
The density-ratio output is parameterized as the square of the unconstrained
network output and is therefore nonnegative.

For GenDICE, the primal regularizer is $1$, the dual regularizer is $0$, and
the reward term is omitted from the Bellman objective. For BestDICE, the
primal regularizer is $0$, the dual regularizer is $1$, and the reward term is
included. Both methods use a normalization regularizer equal to $1$.

The logged transitions are sampled uniformly without explicit $\gamma^t$
weighting. For Env I and Env II, target-action expectations are computed
exactly by summing over the discrete action space. For the continuous-action
MetaWorld tasks, one target-policy action is sampled per state and update.
The initial-state term uses the same fixed evaluation-state distribution
$\nu$ that defines
\[
J^\pi=\mathbb{E}_{S\sim\nu}[V^\pi(S)].
\]

The final scalar estimate is
\[
\widehat J^\pi
=
\frac{1}{1-\gamma}
\frac{
\sum_i \widehat\zeta(S_i,A_i)R_i
}{
\sum_i \widehat\zeta(S_i,A_i)
}.
\]
No additional $\gamma^t$ factor is used in this final estimator.

GenDICE and BestDICE use the complete logged dataset in each repetition,
giving $2030$ transitions in Env I and Env II and $5015$ transitions in each
MetaWorld task. In contrast, Mixed-AC uses separate critic-construction,
value-fitting, and validation subsets. Ground-truth policy values are used
only after fitting to compute the evaluation errors.

\subsection{Importance-ratio sensitivity}
\label{app:ratio_sensitivity}

We study the effect of target-density approximation and self-normalization at
$\alpha=.5$. For MetaWorld, the primary estimator uses the exact target
density $\pi$ together with the estimated behavior density $\widehat b$. We
compare this with the previous estimated-target approximation, both before
and after self-normalization. This sensitivity experiment uses $20$
repetitions and approximately $2000$ training tuples per repetition.

\begin{table}[h]
\centering
\caption{Mean pointwise MSE in the MetaWorld importance-ratio sensitivity
experiment.}
\label{tab:ratio_sensitivity}
\small
\setlength{\tabcolsep}{3.5pt}
\begin{tabular}{lcccc}
\toprule
& \multicolumn{2}{c}{\textbf{Unnormalized}}
& \multicolumn{2}{c}{\textbf{Self-normalized}}\\
\cmidrule(lr){2-3}\cmidrule(lr){4-5}
& $\widehat\pi/\widehat b$
& $\pi/\widehat b$
& $\widehat\pi/\widehat b$
& $\pi/\widehat b$\\
\midrule
Env III & $123.80$ & $204.38$ & $.2368$   & $.2364$\\
Env IV  & $66.23$  & $82.10$  & $.1219$   & $.1196$\\
Env V   & $4.377$  & $5.756$  & $.003571$ & $.003646$\\
Env VI  & $225.38$ & $280.98$ & $.1202$   & $.1209$\\
\bottomrule
\end{tabular}
\end{table}

The unnormalized ratios produce very large MSE in the MetaWorld tasks,
whereas self-normalization substantially reduces this instability. After
self-normalization, using the exact target density or the previous estimated
target density gives very similar performance. Thus, replacing
$\widehat\pi$ by the analytically known $\pi$ removes the need to estimate
the target density without materially changing the empirical conclusions.

In the main $100$-repetition experiment, Door has the most variable
two-step weights. Its mean two-step effective sample size is approximately
$151$, the mean maximum weight is approximately $28.4$, and about $24\%$
of the total weight is carried by the largest one percent of observations.
The corresponding weights are substantially less concentrated in the other
three MetaWorld tasks.

\subsection{Computational scaling}
\label{app:scaling}

We measure model-fitting time, peak Python-attributed memory, and feature
storage as the sample size increases. Each size is evaluated over $10$
repetitions.

The reported model-fitting time is measured from immediately before the
estimator fit through completion of value fitting. It includes feature
standardization and feature-map construction, conditional-mean network
training, and value fitting; prediction on evaluation states is timed
separately. The smallest sample sizes are affected by fixed initialization
overhead, which explains the nonmonotonic timing between 1k and 2k.

\begin{table}[h]
\centering
\caption{Mean total fitting time in seconds.}
\label{tab:scaling_time}
\small
\setlength{\tabcolsep}{4pt}
\begin{tabular}{lrrrrrr}
\toprule
& 1k & 2k & 5k & 7.5k & 10k & 20k\\
\midrule
Env I   & .673 & .587 & .783 & .940 & 1.078 & --\\
Env II  & .714 & .640 & .935 & 1.185 & 1.453 & --\\
Env III & 1.001 & .726 & .852 & .955 & 1.024 & 1.464\\
Env IV  & .998  & .709 & .802 & .933 & 1.000 & 1.369\\
Env V   & 1.030 & .743 & .850 & .948 & 1.050 & 1.438\\
Env VI  & 1.241 & .978 & 1.196 & 1.408 & 1.571 & 2.285\\
\bottomrule
\end{tabular}
\end{table}

\begin{table}[h]
\centering
\caption{Peak Python-attributed memory in MB.}
\label{tab:scaling_memory}
\small
\setlength{\tabcolsep}{4pt}
\begin{tabular}{lrrrrrr}
\toprule
& 1k & 2k & 5k & 7.5k & 10k & 20k\\
\midrule
Env I   & 11.32 & 11.23 & 27.84 & 41.72 & 55.32 & --\\
Env II  & 13.11 & 16.41 & 40.76 & 61.08 & 81.02 & --\\
Env III & 13.42 & 13.28 & 32.22 & 48.40 & 64.10 & 127.11\\
Env IV  & 13.42 & 12.98 & 31.55 & 47.42 & 62.81 & 124.62\\
Env V   & 13.42 & 12.98 & 31.55 & 47.42 & 62.81 & 124.62\\
Env VI  & 15.76 & 21.86 & 53.18 & 80.00 & 105.97 & 210.27\\
\bottomrule
\end{tabular}
\end{table}

\begin{table}[h]
\centering
\caption{Feature storage in MB.}
\label{tab:scaling_storage}
\small
\setlength{\tabcolsep}{4pt}
\begin{tabular}{lrrrrrr}
\toprule
& 1k & 2k & 5k & 7.5k & 10k & 20k\\
\midrule
Env I   & 2.29 & 4.33 & 10.96 & 16.57 & 21.92 & --\\
Env II  & 3.44 & 6.50 & 16.44 & 24.85 & 32.88 & --\\
Env III & 1.729 & 3.673 & 9.075 & 13.828 & 18.365 & 36.515\\
Env IV  & 1.729 & 3.673 & 9.075 & 13.828 & 18.365 & 36.515\\
Env V   & 1.729 & 3.673 & 9.075 & 13.828 & 18.365 & 36.515\\
Env VI  & 3.457 & 7.346 & 18.149 & 27.656 & 36.731 & 73.030\\
\bottomrule
\end{tabular}
\end{table}

Beyond the smallest sample sizes, model-fitting time, memory, and feature storage
grow approximately linearly over the evaluated range. At approximately
$20{,}000$ tuples, the total fitting time is below three seconds for each
MetaWorld task.

Runtime is measured using \texttt{time.perf\_counter}. Peak memory is
measured using \texttt{tracemalloc} and therefore refers to Python-attributed
allocated memory. Feature storage is calculated from the stored feature
arrays. The synthetic experiments use a fixed data budget of approximately
$2030$ tuples, so larger scaling sizes were not evaluated for Env I and Env II.

\end{document}